\documentclass[10pt,twocolumn,letterpaper]{article}

\usepackage[final]{cvpr}   

\usepackage[utf8]{inputenc}
\usepackage[T1]{fontenc}
\usepackage{booktabs}
\usepackage{amsfonts}
\usepackage{amsmath}
\usepackage{amssymb}
\usepackage{nicefrac}
\usepackage{microtype}
\usepackage{xcolor}
\usepackage{algorithm}
\usepackage{algorithmic}
\usepackage{multirow}
\usepackage{amsthm}
\usepackage{float}

\newtheorem{theorem}{Theorem}[section]
\newtheorem{lemma}[theorem]{Lemma}

\newtheorem{definition}[theorem]{Definition}
\newtheorem{assumption}[theorem]{Assumption}

\usepackage[pagebackref,breaklinks,colorlinks,bookmarks=false]{hyperref}

\def\paperID{****}   
\def\confName{CVPR}
\def\confYear{2026}

\title{Understanding Cross-Modal Contributions in Continual\\
Vision--Language Models: A Theoretical Perspective}

\author{
Salimeh Sekeh\\
San Diego State University\\
{\tt\small ssekeh@sdsu.edu}
\and
Mary Wisell\\
San Diego State University\\
{\tt\small mwisell4707@sdsu.edu }
}

\begin{document}
\maketitle

\begin{abstract}
Continual vision-language models are commonly addressed through sequential fine‑tuning; however, although this paradigm enables adaptation to new environments (tasks), it inherently emphasizes the contribution of previously learned environments (tasks) at the expense of the stability required to preserve previously acquired knowledge. While existing approaches have adequately studied continual learning and catastrophic forgetting in vision–language models (VLMs), the theoretical understanding of modality‑specific contributions across a sequence of environments remains largely unexplored. In this paper, we present a new theoretical perspective to understand the cross-modal (vision-language) contributions to consecutive environments. We empirically evaluate our theoretical findings on large VLMs and demonstrate their effectiveness in capturing environment‑level cross-modal contributions. Our analysis provides deeper insights into continual VLMs, highlighting their contribution robustness to varying task orders and inter‑task similarities, and their improved generalization performance.
\end{abstract}

\section{Introduction}
\label{sec:intro}

Large vision-language models such as CLIP~\citep{radford2021clip} are increasingly adapted to new environments via sequential fine-tuning, enabling continual learning without full retraining.
This paradigm inherently trades plasticity for stability: acquiring representations for a new environment overwrites those learned for earlier ones, causing catastrophic forgetting~\citep{kirkpatrick2017overcoming,mirzadeh2020understanding}.
Existing methods, such as prompt tuning~\citep{wang2022l2p,wang2022dualprompt,zhou2024ease}, adapters~\citep{gao2021clipadapter}, and weight interpolation~\citep{zheng2023zscl}, suppress forgetting empirically, but none formally characterize \emph{which property of the joint image-text embedding space} predicts how much forgetting occurs.

We fill this gap with \emph{modality-specific contribution scores}: pairwise cosine expectations measuring the alignment of image and text features across any two environments.
Four scores capture cross-modal ($Con_{\mathcal{I}_t\to\mathcal{T}_{t'}}$, $Con_{\mathcal{T}_t\to\mathcal{I}_{t'}}$) and intra-modal ($Con_{\mathcal{I}_t\to\mathcal{I}_{t'}}$, $Con_{\mathcal{T}_t\to\mathcal{T}_{t'}}$) structure between environments $\mathcal{E}_t$ and $\mathcal{E}_{t'}$.
Computable from a single frozen forward pass, these scalars jointly determine loss landscape drift, the optimal gradient alignment mask, and the magnitude of catastrophic forgetting.

We develop three theoretical analyses (\textbf{T1}-\textbf{T3}) connecting these scores to forgetting in continual VLMs. \textbf{T1} proves that total loss drift $\Delta_{\rm LOSS}(t\!\to\!t')$ is aligned with cross-model contributions and decomposes linearly into the two cross-modal contributions, holding separately for both the CE and CLIP loss components. 

Empirically, the CLIP component has exactly zero cross-seed variance under a frozen encoder and can reverse the sign of $\Delta_{CE}$ entirely, confirming both are necessary to predict inter-environment transition direction.
\textbf{T2} introduces the \emph{Contribution Balance Factor} (CBF) and proves its optimum is achieved by scaling cross-modal similarity scores. 
The CLIP weight $\lambda$ monotonically increases intra-environment alignment, $C_{LV}(t)$, under adaptive fine-tuning (from ${\approx}0.04$ at $\lambda\!=\!0$ to ${\approx}0.50$ at $\lambda\!=\!5$ on CUB dataset~\citep{wah2011cub} in Section~\ref{sec:setup}). 

The CBF objective reduces analytically to $p^2$ regardless of encoder state or $\lambda$, and the CLIP gradient is exactly zero under a frozen encoder but non-zero under an adaptive one.

\textbf{T3} derives a closed-form forgetting bound. 
On MS-COCO dataset~\citep{lin2014coco} (Section~\ref{sec:setup}), where cross-modal contributions vary substantially across environment pairs, the best of 120 orderings achieves up to $1.74{\times}$ lower total forgetting than the worst, with Pearson correlation between total forgetting and the cumulative contribution cost consistently negative across all three architectures.
On CUB-200-2011, where contributions are near-uniform ($\sigma_I\!=\!0.063$ vs $0.138$ on COCO), the ordering range narrows to $1.16{\times}$, indicating the prediction is most discriminative when cross-modal variation across environments is substantial.

We additionally investigate out-of-distribution (OOD) robustness empirically, finding that contribution-guided ordering reduces catastrophic forgetting and simultaneously improves cumulative OOD detection performance. We evaluate on MS-COCO~2017~\citep{lin2014coco} and CUB-200-2011~\citep{wah2011cub} with three CLIP backbones (RN50, ViT-B/32, ViT-B/16) under frozen and adaptive protocols, three reproducibility seeds each.

\textit{Our contributions are listed below:}
\begin{itemize}[nosep,leftmargin=*,topsep=2pt]

  \item \textbf{Loss decomposition (T1).} We prove that $\Delta_{\rm LOSS}(t\!\to\!t')$ decomposes linearly into cross-modal contributions for both loss components, and confirm that both terms are necessary to predict cross-environment transition direction.

  \item \textbf{Contribution Balance Factor (T2).} We introduce the CBF and prove its optimum reduces to $p^2$ under a contribution-normalized mask. The CLIP weight $\lambda$ monotonically controls $C_{LV}(t)$ under adaptive fine-tuning, with an exact null result under a frozen encoder.

  \item \textbf{Forgetting bound and ordering (T3).} We derive a closed-form forgetting bound and demonstrate up to $1.74{\times}$ forgetting reduction through contribution-guided ordering on COCO, with prediction strength governed by cross-modal variation in the embedding space.
  
  \item \textbf{OOD robustness.} We empirically show that contribution scores predict pairwise OOD discriminability on COCO and that the T3 contribution-guided ordering simultaneously reduces forgetting and improves cumulative OOD AUROC.
  
  \item \textbf{Empirical validation.} Two datasets with contrasting cross-modal geometries, three CLIP architectures, six CLIP loss weights, frozen and adaptive protocols, and 120-ordering sweeps jointly validate the theoretical framework and characterize its boundaries.

\end{itemize}
\section{Related Work}
\label{sec:related_work}

Continual learning mitigates catastrophic forgetting via regularization~\cite{kirkpatrick2017overcoming,zenke2017si,aljundi2018mas}, gradient projection~\cite{lopez2017gradient,farajtabar2020orthogonal,saha2021gpm}, and replay~\cite{rebuffi2017icarl,rolnick2019er}.
In the VLM era, prompt-based methods~\cite{wang2022l2p,wang2022dualprompt,zhou2024ease} adapt CLIP without modifying the frozen backbone, parameter-efficient adapters~\cite{gao2021clipadapter,zhou2022coop,zhou2022cocoop} tune lightweight modules within the embedding space, and ZSCL~\cite{zheng2023zscl} preserves zero-shot transfer via reference-model weight interpolation. 
None formally characterize how modality-specific alignment contributes to forgetting across sequential environments.
Multi-modal continual learning ~\cite{marouf2025quad} extends these ideas beyond classification but similarly lacks a contribution-theoretic framework.
Our work provides this theoretical grounding, deriving closed-form bounds on forgetting in terms of modality-specific alignment scores across sequential environments. OOD detection exploits VLM alignment through MCM~\cite{ming2022delim} and related zero-shot methods~\cite{fort2021clip,wang2023clipn,esmaeilpour2022zero}, while continual OOD methods~\cite{naiknaware2024tempscone,naiknaware2025tqpm} address temporal distribution shift. 
We additionally investigate empirically how contribution scores relate to OOD robustness in the sequential setting. Extended discussions are in Appendix~\ref{app:related_work}.

\section{Preliminaries}
\label{sec:prelim}

\paragraph{Dynamic Vision-Language Data.} The data model $\mathcal{D}=\{{\bf x}_t,y_i\}_{i=1}^n$ from environments $\mathcal{E}:=\{\mathcal{E}_1,\mathcal{E}_2,\ldots, \mathcal{E}_T\}$ is provided. At each time step $t$, the model encounters $\mathcal{D}_t=\{\mathbf{x}_{t,i},y_{t,i}\}_{i=1}^{n_t}$ from environment $\mathcal{E}_t$. At each time step $t$, in environment $\mathcal{E}_t$, each data consist of both vision and language modality samples $\mathcal{D}_t=\{\mathcal{I}_t,\mathcal{T}_t\}$, where $\mathcal{I}_t=\{I_{t,i},y_{t,i}\}_{i=1}^{n_t}$ and $\mathcal{T}_t=\{T_{t,i}\}_{i=1}^{n_t}$.

\noindent{\it Text Pattern Construction:} If the ID images are not accompanied with captions, we construct text patterns using the text encoder of the frozen VLM. At time step $t$, for each class $y^{id}_t\in\{1,\ldots,C_t\}$, employ prompt ensembling over $P$ templates to obtain robust text representations. Let $\{p_{c_t}^{(1)}, \dots, p_{c_t}^{(P)}\}$ denote the set of prompts for class $c_t$ in environment $\mathcal{E}_t$. The class text embedding is computed as:
\begin{align}\label{eq:class-text-embedding}
\mathcal{T}_{c_t} = \mathrm{Normalize}\!\left(\sum_{i=1}^{P}
\mathrm{Normalize}\!\left(\phi^T(p_{c_t}^{(i)})\right)\right),
\end{align}
where $\phi^T$ denotes the frozen CLIP text encoder and
normalization is the $\ell_2$-norm. Collecting all ID class embeddings yields
the \textit{ID text bank}
$\mathcal{T}_t^{id} = [{T}_{t,1}, \dots, T_{t,C_t}]^\top \in \mathbb{R}^{C_t \times d}$, where $d$ is the embedding dimension. Each image $I^{id}_{t,i}$ is paired with a text from the ID text bank. $\mathcal{T}_t^{id}$ is computed
at each time step and varies over next time step based on new class labels. \\
If the images are already paired with captions $\mathcal{T}_t^{id} = [{T}_{t,1}, \dots, T_{t,n_t}]\in \mathbb{R}^{n_t \times d}$. \\
\noindent{\it Feature Representations:} CLIP contains an image encoder $f (\cdot)$ and a text encoder $g(\cdot)$, designed to extract features from images and text descriptions respectively. Given the pair (image, text) $(I_{t,i}, T_{t,i})$, the image can
be encoded in $z_{t,i} = f(I_{t,i}) \in \mathbb{R}^{d\times 1}$ and the text can be encoded in $u_{t,i} = g(T_{t,i}) \in \mathbb{R}^{d\times 1}$.

\noindent{\it OOD Detection:} During the inference, at each timestep $t$, new environment ($\mathcal{E}_t$) with unlabeled $\mathcal{D}_t^{id}$ and $\mathcal{D}_t^{ood}$ are received. The model predicts the label of $\mathcal{D}_t^{id}$ and rejects out-of-distribution (OOD) examples i.e. $\mathcal{D}_t^{ood}$. Note that 
$\mathcal{D}_t^{id}=\{\mathcal{I}^{id}_t,\mathcal{T}^{id}_t\}$ and $\mathcal{D}_t^{ood}=\{\mathcal{I}^{ood}_t,\mathcal{T}^{ood}_t\}$ are collection of vision and text pairs. At each time step $t$, the model encounters multi-modal data drawn from a non-stationary distribution
\begin{equation}
\mathbb{P}_t= \pi_{t}^{id}\;\mathbb{P}_t^{{id}}
+ \pi_t^{ood}\mathbb{P}_t^{{ood}},\;\;\;\hbox{where}\;\;\;\pi_{t}^{id}+\pi_t^{ood}=1.
\label{eq:wild_method}
\end{equation}
where $\mathbb{P}^{id}$ denotes the in-distribution (ID) data and $\mathbb{P}_t^{ood}$ denotes semantic OOD data that should be rejected. The detector must simultaneously: 
\textit{i)} maintain classification performance on ID samples, and 
\textit{ii)} reduce false positive rate (FPR) or increase area under the receiver operating characteristic (AUROC) on semantic OOD samples.

\paragraph{Loss Function in Terms of Modality Contribution.} Assuming that no OOD data is involved in dataset, the loss function (\ref{Loss-function(ID-OOD)}) below will be simplified to 

\begin{align}\label{Loss-function(ID)}
\mathcal{L}_{CE}(\theta_t):=\;&\mathbb{E}_{(\mathbf{x}_t^{id},y_t)\sim P^{id}}[\mathcal{L}_{CE}(f_{\theta_t}(\mathbf{x}_t),y_t)]\nonumber\\
=\;&-\frac{1}{n_t}\sum_{i=1}^{n_t}\log \frac{e^{z^{id}_{t,i}\cdot u^{id}_{t,y_i}/\tau}}{\sum_{c=1}^C e^{z^{id}_{t,i}\cdot u^{id}_{t,c}/\tau}}, 
\end{align}

where $u^{id}_{t,c}$ is the text embedding for class label $c$ in ID dataset and $\tau$ is the temperature parameter. 
The CLIP loss is a sum of two cross entropy losses. The first cross entropy term is the image-to-text loss, which is minimized when the text embedding corresponding to an image has a high dot product with that image. The second term inverts this relationship to form a text-to-image loss.
\begin{align}\label{Loss-function(CLIP)}
\mathcal{L}_{CLIP} (\theta_t)=\frac{1}{2n_t}\sum_{i=1}^{n_t}\bigg[
&-\log \frac{e^{z^{id}_{t,i}\cdot u^{id}_{t,i}/\tau }}{\sum_{j=1}^{n_t} e^{z^{id}_{t,i}\cdot u^{id}_{t,j}/\tau }}\nonumber\\
&-\log\frac{e^{u^{id}_{t,i}\cdot z^{id}_{t,i}/\tau }}{\sum_{j=1}^{n_t} e^{u^{id}_{t,i}\cdot z^{id}_{t,j}/\tau }}\bigg].
\end{align}
Combining two losses (\ref{Loss-function(ID)}) and (\ref{Loss-function(CLIP)}), we have loss function (\ref{total-loss}) with regularization parameter $\lambda_t$. 
\begin{align}\label{total-loss}
\mathcal{L}(\theta_t)= \mathcal{L}_{CE}(\theta_t)+\lambda_t\; \mathcal{L}_{CLIP} (\theta_t). 
\end{align}wo time steps $t$ and $t'$ ($t'>t$). Assuming that $\lambda_t=\lambda_{t'}=\lambda$ 
\begin{multline}\label{Delta:Loss}
\Delta_{LOSS}(t\rightarrow t'):=\mathcal{L}(\theta_t)-\mathcal{L}(\theta_{t'})\\
= \Delta_{CE}(t\rightarrow t')+\lambda\;\Delta_{CLIP}(t\rightarrow t'), 
\end{multline}
where $\Delta_{CE}(t\rightarrow t'):=\mathcal{L}_{CE}(\theta_t)-\mathcal{L}_{CE}(\theta_{t'})$ and $\Delta_{CLIP}(t\rightarrow t'):=\mathcal{L}_{CLIP}(\theta_t)-\mathcal{L}_{CLIP}(\theta_{t'})$. 

\paragraph{Loss Function with OOD Detection Regularization.} 

At each time step $t$, we have parameter set $\theta_t$ to be learned using general loss function $\mathcal{L}_t$ and data $(\mathbf{x}^{id},y)$ with distribution $ P^{id}$ and $\mathbf{x}^{ood}$ with semantic distribution shift (OOD) distribution $P^{ood}$.
\begin{align}\label{Loss-function(ID-OOD)}
\mathcal{L}(\theta_t)=\;&\mathbb{E}_{(\mathbf{x}_t^{id},y_t)\sim \mathbb{P}^{id}}[\mathcal{L}_{LOSS}(f_{\theta_t}(\mathbf{x}_t),y_t)] \nonumber\\
&+\lambda_{ood} \;\mathbb{E}_{\mathbf{x}_t^{ood}\sim {\mathbb{P}^{ood}}}[\mathcal{L}_{reg}(f_{\theta_t} (\mathbf{x}_t^{ood})], 
\end{align}
where $\mathcal{L}_{LOSS}$ is the combination of standard cross entropy $\mathcal{L}_{CE}$ loss and CLIP loss $\mathcal{L}_{CLIP}$ for the original classification task. $\mathcal{L}_{reg}$ is the OOD
detection regularization term depending on the detector used, which generally encourages a high uncertainty on $P^{sem}$. 

For both loss function (\ref{total-loss}) and (\ref{Loss-function(ID-OOD)}), the optimization solves:
$$
\theta^*_t={\rm \arg\min}_{\theta_t} \mathcal{L}(\theta_t).
$$
Throughout Section~\ref{sec:theory}, we analyze the behavior of the total loss in (\ref{total-loss}) using in-distribution (ID) inputs and develop a theoretical framework based on cross-modal contributions and total loss alignment. In the experimental section, where we examine generalization and OOD robustness, we instead employ the loss function with OOD regularization defined in (\ref{Loss-function(ID-OOD)}).
\section{Theory}
\label{sec:theory}

We analyze in continual VLMs over a sequence of environments $\mathcal{E}:=\{\mathcal{E}_1,\mathcal{E}_2,\ldots, \mathcal{E}_T\}$ how the intra-modality and cross-modality contributions is related to the loss function difference and bridges the gap between environments in CL performance and forgetting. Further, we analyze the modality contributions when the VLM encounters OOD examples. We state the intra- and cross-modal contribution measurements in our analysis. 

\begin{definition} \label{def:Intra-Modality Contribution}
(Intra-Modal Contribution) The image $\mathcal{I}^{id}_t$ from environment $\mathcal{E}_t$ with class $C_t$, $\sigma_{\mathcal{I}}$-contributes in feature representation of images $\mathcal{I}^{id}_{t'}$ in environment $\mathcal{E}_{t'}$ if and only if
\begin{align}
Con_{\mathcal{I}_t\rightarrow\mathcal{I}_{t'}}:=\;&\mathbb{E}_{(\mathbf{x}^{id}_{t'},y_{t'})\sim P^{id}}\left[cos(\tilde{z}^{id}_{t'}\cdot z^{id}_{t})\right]\nonumber\\
=\;&(1-\sigma_{\mathcal{I}}),\;\; (t<t').
\end{align}
Similarly the text $\mathcal{T}^{id}_t$ from environment $\mathcal{E}_t$ with class $C_t$, $\sigma_{\mathcal{T}}$-contributes in feature representation of text $\mathcal{T}^{id}_{t'}$ in environment $\mathcal{E}_{t'}$ if and only if
\begin{align}
Con_{\mathcal{T}_t\rightarrow\mathcal{T}_{t'}}:=\;&\mathbb{E}_{(\mathbf{x}^{id}_{t'},y_{t'})\sim P^{id}}\left[cos(\tilde{u}^{id}_{t'}\cdot u^{id}_{t})\right]\nonumber\\
=\;&(1-\sigma_{\mathcal{T}}),\;\;(t<t').
\end{align}
\end{definition}
\begin{definition} \label{def:Cross-Modality Contribution}
(Cross-Modal Contribution) The image $\mathcal{I}^{id}_t$ from environment $\mathcal{E}_t$ with class $C_t$, $\sigma_{\mathcal{IT}}$-contributes in feature representation of text $\mathcal{T}^{id}_{t'}$ in environment $\mathcal{E}_{t'}$ if and only if
\begin{align}
Con_{\mathcal{I}_t\rightarrow\mathcal{T}_{t'}}:=\;&\mathbb{E}_{(\mathbf{x}^{id}_{t'},y_{t'})\sim P^{id}}\left[cos(\tilde{u}^{id}_{t'}\cdot z^{id}_{t})\right]\nonumber\\
=\;&(1-\sigma_{\mathcal{IT}}), \;\; (t<t'),
\end{align}
and the text $\mathcal{T}^{id}_t$ from environment $\mathcal{E}_t$ with class $C_t$, $\sigma_{\mathcal{TI}}$-contributes in feature representation of images $\mathcal{I}^{id}_{t'}$ in environment $\mathcal{E}_{t'}$
\begin{align}
Con_{\mathcal{T}_t\rightarrow\mathcal{I}_{t'}}:=\;&\mathbb{E}_{(\mathbf{x}^{id}_{t'},y_{t'})\sim P^{id}}\left[cos(u^{id}_{t}\cdot\tilde{z}^{id}_{t'})\right]\nonumber\\
=\;&(1-\sigma_{\mathcal{TI}}),\;\; (t<t').
\end{align}
\end{definition}
In both Def.~\ref{def:Intra-Modality Contribution} and \ref{def:Cross-Modality Contribution}, $cos(\cdot,\cdot)$ is cosine similarity and $0\leq\sigma_{\mathcal{I}}, \sigma_{\mathcal{T}},\sigma_{\mathcal{IT}},\sigma_{\mathcal{TI}}\leq1$.\\
In this section, we provide four main theoretical analysis ({\bf T1-T3})  of the vision-language contributions in continual learning, highlighting the impact of modality contributions in learning of a sequence of environments. 
\\
\noindent\textbf{T1: $\Delta_{Loss}(t\rightarrow t')$ and cross-modal contributions alignment.}

In T1 analysis, we theoretically explain that the overall loss difference (\ref{Delta:Loss}) linearly changes in terms of cross-modal contributions $Con_{\mathcal{I}_t\rightarrow\mathcal{T}_{t'}}$ and $Con_{\mathcal{T}_t\rightarrow\mathcal{I}_{t'}}$. To this end, we develop Lemma~\ref{lem.1} and Lemma~\ref{lem.2} which show that $\Delta^{\square}(t\rightarrow t')$, $\square=\{CE,CLIP\}$ are aligned with cross-modal contributions. Here we assume that $n_t=n_{t'}=n$. 
\begin{assumption}[Environment Smooth Transition]
\label{ass:smooth} Consider two environments $\mathcal{E}_t$ and $\mathcal{E}_{t'}$, we assume that 
 $\exists$ constants  $\beta_{u,i}, \beta_{z,i}\geq 0$ such that the text embeddings $u^{id}$ are bridged as $u^{id}_{t,y_i}=(1+\beta_{u,i})\;\tilde{u}^{id}_{t',\tilde{y}_i}$, and the image embedding $z^{id}$ are related as $z^{id}_{t,i}=(1+\beta_{z,i})^{-1}\;\tilde{z}^{id}_{t,i}$. 
\end{assumption}
Lemma~\ref{lem.1} and Lemma~\ref{lem.2} serve a constructive analysis of $\mathcal{L}_{CE}$ and $\mathcal{L}_{CLIP}$, respectively supporting our main theorem in T1, as follows.

\begin{lemma}\label{lem.1}
Under the Assumption~\ref{ass:smooth},  the cross-entropy loss difference at time steps $t$ and $t'$ (i.e $\Delta_{CE}(t\rightarrow t'):=\mathcal{L}_{CE}(\theta_t)-\mathcal{L}_{CE}(\theta_{t'})$) is a linear regression of cross-modal contributions $Con_{\mathcal{I}_t\rightarrow\mathcal{T}_{t'}}$ and $Con_{\mathcal{T}_t\rightarrow\mathcal{I}_{t'}}$ as
\begin{align}\label{main:lem-1}
\Delta_{CE}(t\rightarrow t')=-\frac{1}{2\tau}\bigg[&\bar{\kappa}_1 \; Con_{\mathcal{I}_t\rightarrow\mathcal{T}_{t'}}+\bar{\kappa}_2 \; Con_{\mathcal{T}_t\rightarrow\mathcal{I}_{t'}}\nonumber\\
&+\frac{2\tau}{n}\sum_{i=1}^n\log(\alpha^C_{t',i}/\alpha^C_{t,i})\bigg], 
\end{align}
where $\tau$ is temperature, $\bar{\kappa}_1,\bar{\kappa}_2\neq 0$ are constants and $\alpha^C_{t',i}, \alpha^C_{t,i}$ are normalizations in $\mathcal{L}_{CE}$, (\ref{Loss-function(ID)}). 
\end{lemma}

\begin{lemma} \label{lem.2}
Under the Assumption~\ref{ass:smooth}, the CLIP loss difference at time steps $t$ and $t'$ (i.e $\Delta_{CLIP}(t\rightarrow t'):=\mathcal{L}_{CLIP}(\theta_t)-\mathcal{L}_{CLIP}(\theta_{t'})$) is a linear regression of cross-modal contributions $Con_{\mathcal{I}_t\rightarrow\mathcal{T}_{t'}}$ and $Con_{\mathcal{T}_t\rightarrow\mathcal{I}_{t'}}$: 
\begin{align}\label{main:lem-2}
\Delta_{CLIP}(t\rightarrow t')= \frac{1}{2\tau}\bigg[\kappa_1 Con_{\mathcal{I}_t\rightarrow\mathcal{T}_{t'}}+{\kappa}_2 \; Con_{\mathcal{T}_t\rightarrow\mathcal{I}_{t'}}\nonumber\\
+\frac{1}{n}\sum_{i=1}^n\tau\left(\log \alpha^{z}_{t,i}/\tilde{\alpha}^{\tilde{z}}_{t',i}+\log \alpha^{u}_{t,i}/\tilde{\alpha}^{\tilde{u}}_{t',i}\right)\bigg].
\end{align}
where $\tau$ is temperature, ${\kappa}_1,{\kappa}_2\neq0$ are constants and $\alpha^{z}_{t,i}, \tilde{\alpha}^{\tilde{z}}_{t',i}, \alpha^{u}_{t,i}, \tilde{\alpha}^{\tilde{u}}_{t',i}$ are normalizations in $\mathcal{L}_{CLIP}$, (\ref{Loss-function(CLIP)}). 
\end{lemma}
The theorem below follows directly from Lemmas~\ref{lem.1} and~\ref{lem.2}, with the full proof provided in Appendix~\ref{app:disscuss_thm_1}.
\begin{theorem}\label{T1:loss-difference}
Under the Assumption~\ref{ass:smooth}, the loss difference at time steps $t$ and $t'$ (i.e $\Delta_{LOSS}(t\rightarrow t'):=\mathcal{L}(\theta_t)-\mathcal{L}(\theta_{t'})$) is a linear regression of cross-modal contributions as $\Delta_{Loss}(t\rightarrow t')=\beta_1Con_{\mathcal{I}_t\rightarrow\mathcal{T}_{t'}}+\beta_2Con_{\mathcal{T}_t\rightarrow\mathcal{I}_{t'}}+\alpha$ where $\beta_1$ and $\beta_2$ are regression multipliers and $\alpha$ is the offset. 
\end{theorem}
\noindent{\bf Remark:} If the image $\mathcal{I}^{id}_t$ from environment $\mathcal{E}_t$ with class $C_t$, $\sigma_{\mathcal{IT}}$-contributes in feature representation of text $\mathcal{T}^{id}_{t'}$ in environment $\mathcal{E}_{t'}$ and the text $\mathcal{T}^{id}_t$ from environment $\mathcal{E}_t$ with class $C_t$, $\sigma_{\mathcal{TI}}$-contributes in feature representation of images $\mathcal{I}^{id}_{t'}$ in environment $\mathcal{E}_{t'}$, then we computer loss drift between $\mathcal{E}_t$ and $\mathcal{E}_{t'}$ by $\Delta_{Loss}(t\rightarrow t')=\beta_1 (1-\sigma_{\mathcal{IT}})+\beta_2(1-\sigma_{\mathcal{TI}})+\alpha.$ \\
\noindent{\bf Discussion.} 
Theorem~\ref{T1:loss-difference} establishes a direct connection between total loss drift and modality-specific contributions. This implies that the stability between $\mathcal{E}_t$ and $\mathcal{E}_{t'}$ can be regulated through the regression coefficients $\beta_1$, $\beta_2$ and offset $\alpha$. Moreover, to prevent performance degradation when transitioning to more complex environments, particularly when text-to-image and image-to-text contributions cross-environments are not well controlled, these regression parameters can help stabilize and recover the learning loss. Section~\ref{sec:experiments} provides experimental observations on regression parameters. 
\\
\\
\noindent\textbf{T2: Cross-modality contribution balance factor bridges the gap between environments.}

\noindent We next explain that by controlling the balance of cross-modality contributions impacts the learning transition between environments. We mainly seek to investigate the central question:
\begin{center}
\textit{How modal contribution balance factor controls intra- and cross-modal contribution between environments?}
\end{center}
\begin{definition}\label{def:CBF}
(Cross-Modal Contribution Balance Factor) We define contribution balance factor (CBF) from $\mathcal{E}_{t}$ to $\mathcal{E}_{t'}$  by $\mathbf{P}_{t\rightarrow t'}\in\{(0,1)\}^d$ which is optimized by
\begin{align*}
\mathbf{P}_{t\rightarrow t'}^*={\rm argmax}_{\mathbf{P}_{t\rightarrow t'}\in\{(0,1)\}^d}\;\frac{\|\mathbf{P}_{t\rightarrow t'}\odot\nabla_\theta \mathcal{L}(\theta_{t'})\|^2}{\|\nabla_\theta \mathcal{L}(\theta_{t'})\|^2}, \\
\hbox{where $\odot$ is pairwise multiplication.}
\end{align*}
We define the cross-modal CBT from $\mathcal{E}_{t}$ to sample $(\mathbf{x}_{t',i},y_{t',i})$ as $\mathbf{P}^{(i)}_{t\rightarrow t'}$. 
\end{definition}
The assumptions below play distinct roles: Assumptions~\ref{ass:Sufficient class-normalized gradient drift} and~\ref{ass:Sufficient cross-modal gradient drift} focus on the normalized components of $\mathcal{L}_{CE}$ and $\mathcal{L}_{CLIP}$, assuming that their gradient drift is sufficiently large, whereas Assumption~\ref{ass:slow-gradient-magnitude} bounds the overall gradient magnitude of the loss function.
\begin{assumption}[Sufficient class-normalized gradient drift]\label{ass:Sufficient class-normalized gradient drift}
We assume that the cross-modal CBF is sufficiently effective in average over all labels:
\begin{align}
\mathbf{P}_{t\rightarrow t'}\odot\nabla_\theta \log\alpha^C_{t',i}=\;&\Omega(\xi_{t'}), \nonumber\\
&\hbox{where $\alpha^C_{t,i}:=\sum_{c=1}^C e^{z^{id}_{t,i}\cdot u^{id}_{t,c}/\tau}$. }
\end{align}
\end{assumption}
\begin{assumption}[Sufficient cross-modal gradient drift]\label{ass:Sufficient cross-modal gradient drift}
 We assume that the cross-modal CBF is sufficiently effective in average over all samples.
\begin{align}
&\frac{1}{n_{t'}}\sum_{i=1}^{n_{t'}} \mathbf{P}_{t\rightarrow t'}\odot\nabla_\theta \log\alpha^{z}_{t',i}=\Omega(\xi_{z,t'}), \nonumber\\
&\hbox{where $\alpha_{t,i}^{z}:= \sum_{j=1}^{n_t} e^{z^{id}_{t,i}\cdot u^{id}_{t,j}/\tau }$. }\\
&\frac{1}{n_{t'}}\sum_{i=1}^{n_{t'}} \mathbf{P}_{t\rightarrow t'}\odot\nabla_\theta \log\alpha^{u}_{t',i}=\Omega(\xi_{u,t'}), \nonumber\\
&\hbox{where $\alpha_{t,i}^{u}:= \sum_{j=1}^{n_t} e^{u^{id}_{t,i}\cdot z^{id}_{t,j}/\tau }$. }
\end{align}
\begin{assumption}[Slow gradient]
\label{ass:slow-gradient-magnitude}
The magnitude of the gradient of loss changes slowly i.e.
$\|\nabla_\theta \mathcal{L}(\theta_{t})\|^2=\mathcal{O}(\xi)$ which implies $1/\|\nabla_\theta \mathcal{L}(\theta_{t})\|^2=\Omega(\xi)$. 
\end{assumption}
\end{assumption}
\begin{theorem}[Optimal CBF can be achieved by applying BF on cross-modality similarity scores.]\label{thm2:BF-Gradient}
Under Assumptions~\ref{ass:Sufficient class-normalized gradient drift}-\ref{ass:slow-gradient-magnitude}, the optimal $\mathbf{P}_{t\rightarrow t'}^*$ is alternatively achieved by 
\begin{align}\label{thm:CBF-optimization}
\mathbf{P}_{t\rightarrow t'}^*=\frac{\Omega(\xi)}{2n_{t'}}\sum_{i=1}^{n_{t'}} argmax_{\mathbf{P}^{(i)}_{t\rightarrow t'}} S^{id}_t [\mathbf{P}^{(i)}_{t\rightarrow t'}],  
\end{align}
where $\xi$ is constant and  
\begin{align}\label{eq:cbf-score}
S^{id}_t [\mathbf{P}^{(i)}_{t\rightarrow t'}]:=\;&\|\mathbf{P}^{(i)}_{t\rightarrow t'}\odot\nabla_\theta S^{id}_{t',y_i}\|^2\nonumber\\
&+\lambda\|\mathbf{P}^{(i)}_{t\rightarrow t'}\odot\nabla_\theta S^{id}_{t',z_i\rightarrow u_i}\|^2\nonumber\\
&+\lambda\|\mathbf{P}^{(i)}_{t\rightarrow t'}\odot\nabla_\theta S^{id}_{t',u_i\rightarrow z_i}\|^2,
\end{align}
and intra-environment alignment scores $S^{id}_{t',y_i}$, $S^{id}_{t',z_i\rightarrow u_i}$, and $S^{id}_{t',u_i\rightarrow z_i}$ are defined as follows
\begin{align}\label{eq:scores-intra-environment}
S^{id}_{t,y_i}=(z^{id}_{t,i}\cdot u^{id}_{t,y_i})/\tau, \;\; S^{id}_{t,z_i\rightarrow u_i}&=(z^{id}_{t,i}\cdot u^{id}_{t,i})/\tau,\nonumber\\
\hbox{and}\;\;\; S^{id}_{t,u_i\rightarrow z_i}&=(u^{id}_{t,i}\cdot z^{id}_{t,i})/\tau.
\end{align}
\end{theorem}
\noindent{\bf Remark.} Note that the scores in (\ref{eq:scores-intra-environment}) differ from the cross-modal contributions defined in Def.~\ref{def:Cross-Modality Contribution}, even though both are based on dot products. Intra-environment alignment captures relationships within a single environment and does not aggregate information across environments. In contrast, our cross-modal contribution scores bridge vision and language embeddings across evolving environments, facilitating the transfer of cross-modal information to subsequent learning stages.
\\
\\
\noindent\textbf{T3: Strong intra- and cross-modality contribution enhances forgetting.} 

A core challenge in continual learning is catastrophic forgetting, also known as catastrophic interference. This phenomenon occurs when neural networks rapidly lose previously acquired knowledge as they learn new information~\cite{mirzadeh2020understanding,wang2024comprehensive, zhou2025learning}. In this section, we analyze forgetting through the lens of environment changes and modality contributions.  

Let $\theta^*_{t}$ and $\theta^*_{t+1}$ be the convergent or optimum parameters after training has been finished for $\mathcal{E}_{t}$ and $\mathcal{E}_{t+1}$ sequentially. We formally define the forgetting (of the first environment, $\mathcal{E}_{t}$) as
\begin{align}\label{def:forgetting}
F_{t}:= \mathcal{L}_t(\theta^*_{t+1})-\mathcal{L}_t(\theta^*_{t})
\end{align} 
\begin{assumption} [monotonic gradient of normalization terms]\label{ass:monotonic gradient of normalization terms} We assume that the gradient of normalization terms in $\mathcal{L}_{CE}$ and $\mathcal{L}_{CLIP}$ is monotonic i.e. the Hessian of normalization terms is zero matrix. 
    
\end{assumption}
\begin{theorem}[Forgetting bound as a function of cross-modal contribution within an environment]\label{thm3:Forgetting-Contribution} Under Assumption~\ref{ass:monotonic gradient of normalization terms}, we bound forgetting by the cross-modal contribution for a single environment $\mathcal{E}_t$ that is
\begin{align}
F_t\leq -\frac{1}{2n_t}\sum_{i=1}^{n_t} \lambda^{min}_{i}\|\theta^*_{t+1}-\theta^*_{t}\|^2,
\end{align}
where $\lambda^{min}_t$ is the minimum eigenvalue of the Hessian $\nabla^2 \left(S^{id}_{t,y_i}+ 2\lambda S^{id}_{t,z_i\rightarrow u_i}+2\lambda S^{id}_{t,u_i\rightarrow z_i}\right)$, for the alignment scores defined in (\ref{eq:scores-intra-environment}).  
\end{theorem}
\begin{assumption}[Compactness]
\label{ass:Compactness}
For two sequential environments $\mathcal{E}_t\rightarrow \mathcal{E}_{t'}$, optimal parameters difference changes within a ball with radius $\delta_{t\rightarrow t'}$, i.e. $\theta_{t'}^*-\theta_t^*\in\mathcal{B}(\theta, \delta_{t\rightarrow t'})$ .
\end{assumption}
\begin{theorem}[Forgetting and cross-modal contribution bridge]\label{thm4:Forgetting-Cross-Modality Contribution} Consider three environments $\mathcal{E}_{t_1}\rightarrow\mathcal{E}_{t_2}\rightarrow \mathcal{E}_{t_3}$ are learned sequentially. Under Assumptions~\ref{ass:Compactness} and ~\ref{ass:smooth}, the forgetting difference between $F_{t_1}$ and $F_{t_2}$, when the VLM has learned the third environment $\mathcal{E}_{t_3}$, is upper bounded by the Hessian of cross-modal contributions from $\mathcal{E}_{t_1}$ to $\mathcal{E}_{t_2}$. For 
\begin{align*}
F_{t_1}=\mathcal{L}_{t_1}(\theta^*_{t_3})-\mathcal{L}_{t_1}(\theta^*_{t_1}),\;\;F_{t_2}=\mathcal{L}_{t_2}(\theta^*_{t_3})-\mathcal{L}_{t_2}(\theta^*_{t_2}).
\end{align*}
\begin{align}\label{main:thm4}
F_{t_1}-F_{t_2}\leq\;&\frac{1}{2} (\theta^*)^\top \beta_{max}\nonumber\\
&\times\nabla^2[Con_{\mathcal{I}_{t_1}\rightarrow\mathcal{T}_{t_2}}+Con_{\mathcal{T}_{t_1}\rightarrow\mathcal{I}_{t_2}}] \theta^*, 
\end{align}
where $\theta^*\in \mathcal{B}(\theta, \delta_{max})$, where $\delta_{max}$ is maximum compactness of optimal parameter differences in $\mathcal{E}_{t_1}\rightarrow \mathcal{E}_{t_3}$ and $\mathcal{E}_{t_2}\rightarrow \mathcal{E}_{t_3}$ transitions. In (\ref{main:thm4}), $\beta_{max}$ is a constant.   
\end{theorem}
\noindent{\bf Discussion.} Theorem~\ref{thm4:Forgetting-Cross-Modality Contribution} provides the central theoretical foundation for the forgetting behavior of learning a sequence of environments.  Specifically, it shows that the cross-modal contributions' behavior impacts on how forgetting varies in environment transitions. Under the assumption that $Con_{\mathcal{I}_{t_1}\rightarrow\mathcal{T}_{t_2}}$, $Con_{\mathcal{T}_{t_1}\rightarrow\mathcal{I}_{t_2}}$ are linear  with respect to $\theta$, forgetting of $\mathcal{E}_{t_1}$ degrades compared to forgetting of $\mathcal{E}_{t_1}$ after learning $\mathcal{E}_{t_3}$. Secondly the forgetting difference can be upper bounded by $F_{t_1}-F_{t_2}\leq \mathcal{O}(\lambda_{max}\delta_{max})$, where $\lambda_{max}$ is maximum eigenvalue of $\nabla^2[Con_{\mathcal{I}_{t_1}\rightarrow\mathcal{T}_{t_2}}+Con_{\mathcal{T}_{t_1}\rightarrow\mathcal{I}_{t_2}}]$. Theorem~\ref{thm3:Forgetting-Contribution} on the other hand focuses on forgetting cross-environment transition. The forgetting is bounded by the Hessian of intra-environment contributions. This means for a fix $\lambda$, if least curvature of the similarity functions $\lambda_{min}$ grows in any direction, the forgetting takes larger values. 


\noindent{\bf Remark.} All theorem proofs are deferred to the Appendix, as the main body in the sequel focuses on experimental observations.
\section{Experiments}
\label{sec:experiments}

\subsection{Datasets and Setup}
\label{sec:setup}

\noindent\textbf{MS-COCO 2017~\cite{lin2014coco}.}
The 80 COCO categories (excluding \emph{person}) are partitioned by majority-vote object assignment into five semantically coherent environments:
(1)~Vehicles \& travel (16 cats.),
(2)~Animals \& nature (16 cats.),
(3)~Food \& kitchen (15 cats.),
(4)~Indoor \& electronics (16 cats.),
(5)~Sports \& personal (16 cats.).
Per-category under-sampling caps each category at $n_{\min}$ images, preventing class-frequency bias in the $\sigma$-contribution matrices. Each image uses its four human-written COCO captions~\cite{lin2014coco}.

\noindent\textbf{CUB-200-2011~\cite{wah2011cub}.}
All 200 bird species are visually similar (bird in flight or perched), keeping intra-modal visual contributions, $C_V(t)$, approximately constant across environments. Any variation in $\Delta_{\rm LOSS}$ is therefore attributable to cross-modal alignment, $C_{LV}(t)$, making CUB a controlled complement to COCO. The 200 species form five taxonomic environments (${\sim}40$ species each): (1)~Waterbirds/Seabirds, (2)~Raptors/Aerial, (3)~Corvids/Perching, (4)~Warblers/Sparrows, (5)~Woodpeckers/Mixed.
Full dataset statistics and CUB-200 caption construction details are provided in Appendix~\ref{app:datasets}.

\noindent\textbf{Models.}
All experiments use three CLIP~\cite{radford2021clip} backbones loaded via open-clip-torch~\cite{cherti2023openclip} and HuggingFace Hub~\cite{wolf2020hf}: CLIP RN50~\cite{openai2021cliprn50,he2016resnet} ($d{=}1024$), CLIP ViT-B/32~\cite{openai2021clipvitb32,dosovitskiy2021vit} ($d{=}512$), and CLIP ViT-B/16~\cite{openai2021clipvitb16,dosovitskiy2021vit} ($d{=}512$, primary model).

\noindent\textbf{Training protocol.}
\textit{Frozen encoder:} features are pre-cached, and a linear head $h_t\!\in\!\mathbb{R}^{C_t\times d}$ is trained per environment (Adam, lr${=}10^{-3}$, 15 epochs, batch 256).
\textit{Adaptive encoder:} the encoder is jointly fine-tuned with the head (encoder lr${=}5{\times}10^{-5}$, AMP fp16, batch 64). All results report means over 3 seeds with checkpoints at epochs 5, 10, 15.
\subsection{Main Experimental Observations} \label{sec:exp}
\noindent\textbf{T1 Experiments: Loss Decomposition and Contribution Alignment.} Theorem~\ref{T1:loss-difference} predicts $\Delta_{\rm LOSS}(t{\to}t') = \beta_1 Con_{\mathcal{I}_t\rightarrow\mathcal{T}_{t'}} + \beta_2 Con_{\mathcal{T}_t\rightarrow\mathcal{I}_{t'}} + \alpha$, following from Lemmas~\ref{lem.1}-\ref{lem.2}, which show that $\Delta_{CE}$ and $\Delta_{\rm CLIP}$ are each linear in the cross-modal contributions under Assumption~\ref{ass:smooth} (~\eqref{Delta:Loss}).
We test this via direct measurement (Observation~1) and OLS regression (Observation~2).

\noindent\textbf{Observation~1: Loss decomposition.}
\begin{table}[b!]
\caption{$\Delta_{\rm LOSS}\!=\!\Delta_{CE}\!+\!\lambda\Delta_{\rm CLIP}$ (\eqref{Delta:Loss}), ViT-B/16, frozen, $\lambda\!=\!1$ (mean, 3~seeds, epoch~15). $\Delta_{\rm CLIP}$ std\,$=0$ at all epochs on both datasets, validating Assumption~\ref{ass:smooth}. COCO $4{\to}5$: $\Delta_{\rm CLIP}$ reverses the sign of $\Delta_{\rm LOSS}$, both terms necessary.}
\label{tab:loss-decomp}\centering\small\setlength{\tabcolsep}{2pt}
\begin{tabular}{@{}llcc rrr@{}}
\toprule
 & \textbf{Pair}
  & $Con_{\mathcal{I}_t\rightarrow\mathcal{T}_{t'}}$ & $Con_{\mathcal{T}_t\rightarrow\mathcal{I}_{t'}}$
  & $\Delta_{CE}$ & $\Delta_{\rm CLIP}$ & $\Delta_{\rm LOSS}$ \\
\midrule
\multirow{4}{*}{COCO}
& $1{\to}2$ & .290 & .267 & $+$0.707 & $+$0.165 & $+$0.872 \\
& $2{\to}3$ & .272 & .261 & $-$0.806 & $-$0.337 & $-$1.144 \\
& $3{\to}4$ & .279 & .292 & $-$0.740 & $+$0.329 & $-$0.411 \\
& $4{\to}5$ & .270 & .276 & $-$0.086 & $+$0.248 & $+$0.163 \\
\midrule
\multirow{4}{*}{CUB}
& $1{\to}2$ & .290 & .271 & $+$0.011 & $+$0.057 & $+$0.068 \\
& $2{\to}3$ & .343 & .347 & $+$0.023 & $+$0.170 & $+$0.193 \\
& $3{\to}4$ & .313 & .315 & $-$0.110 & $-$0.343 & $-$0.453 \\
& $4{\to}5$ & .321 & .317 & $+$0.060 & $+$0.062 & $+$0.122 \\
\bottomrule
\end{tabular}
\end{table}
$\Delta_{\rm CLIP}$ std\,$=0$ validates Assumption~\ref{ass:smooth}: frozen CLIP geometry makes the loss difference deterministic. $\Delta_{\rm CLIP}$ can reverse the sign of $\Delta_{CE}$ (COCO pair $4{\to}5$) and can dominate $|\Delta_{CE}|$ entirely (CUB $3{\to}4$: $|\Delta_{\rm CLIP}|\!=\!0.343,\,|\Delta_{CE}|\!=\!0.110$), confirming that both terms in Theorem~\ref{T1:loss-difference} are necessary to predict loss-drift direction. Under the adaptive encoder, $|\Delta_{\rm CLIP}|$ stabilizes at 39-73\% of its frozen baseline by epoch~15 with sign preserved.
Assumption~\ref{ass:smooth} holds approximately at the conservative encoder learning rate used.

\noindent\textbf{Observation~2: Regression of $\beta_1$, $\beta_2$.}
\begin{table}[t!]
\caption{OLS fit of Theorem~\ref{T1:loss-difference}: $\hat\Delta_{\rm LOSS}\!=\!\hat\beta_1 Con_{\mathcal{I}_t\rightarrow\mathcal{T}_{t'}}+\hat\beta_2 Con_{\mathcal{T}_t\rightarrow\mathcal{I}_{t'}}+\hat\alpha$ over all 20 directed adjacent pairs. $R^2$: proportion of $\Delta_{\rm LOSS}$ variance explained by the two cross-modal contributions.
  Near-zero $R^2$ (frozen) is expected when contributions have limited range. $R^2$ rises as the adaptive encoder widens variance.
  Opposing $\hat\beta_1$, $\hat\beta_2$ signs follows Lemmas~\ref{lem.1}-\ref{lem.2}.}
\label{tab:regression}\centering\small\setlength{\tabcolsep}{3pt}
\begin{tabular}{@{}llcrrrc@{}}
\toprule
 & \textbf{Enc.} & \textbf{Ep.}
  & $\hat\beta_1$ & $\hat\beta_2$ & $R^2$ & MAE \\
\midrule
\multirow{3}{*}{COCO}
& Frozen & 15 & $+$2.37 & $-$2.93 & $-$0.002 & 0.619 \\
& Adaptive & 5 & $+$4.11 & $-$4.79 & $+$0.046 & 0.341 \\
& Adaptive & 10 & $+$3.65 & $-$4.10 & $+$0.122 & 0.187 \\
\midrule
\multirow{3}{*}{CUB}
& Frozen & 15 & $+$3.15 & $-$3.18 & $+$0.011 & 0.201 \\
& Adaptive & 5 & $-$1.03 & $+$1.19 & $+$0.008 & 0.158 \\
& Adaptive & 10 & $-$5.04 & $+$5.11 & $+$0.162 & 0.084 \\
\bottomrule
\end{tabular}
\end{table}
Regression $R^2$ rises from frozen to adaptive consistently across both datasets and all three architectures, confirming Theorem~\ref{T1:loss-difference}: the linear dependence on contributions becomes visible once encoder specialization widens their dynamic range.
Opposing signs of estimated $\hat\beta_1$, $\hat\beta_2$ match the counteracting-modality structure of our Lemmas~\ref{lem.1} and \ref{lem.2}.\\
\begin{figure*}[!h]
\centering
\includegraphics[width=\linewidth]{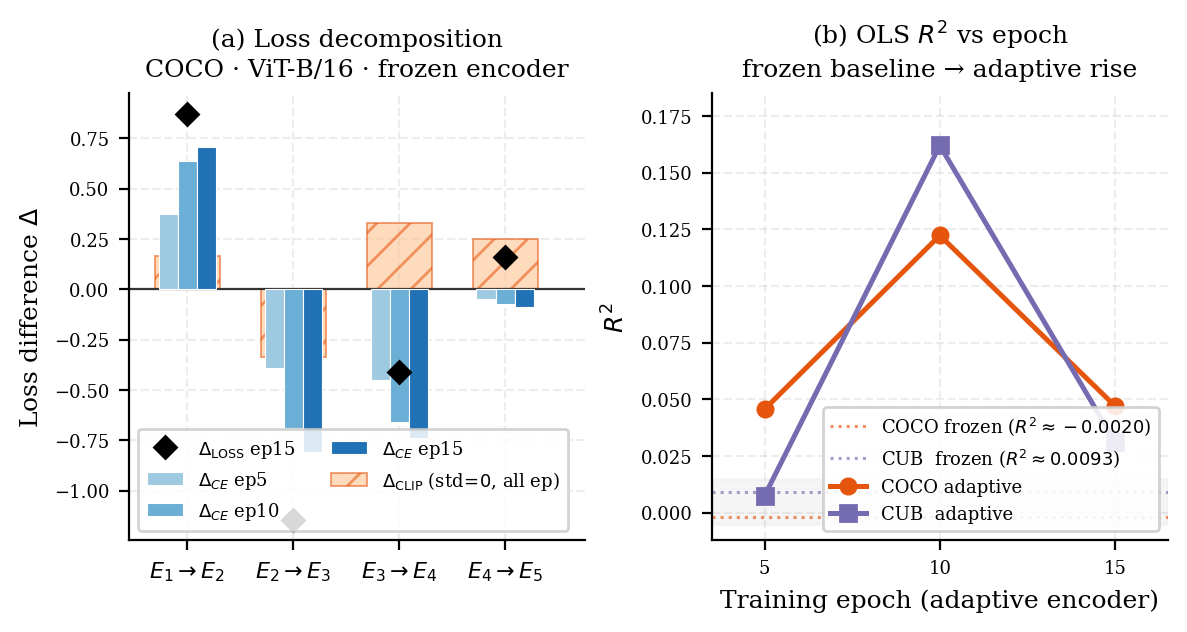}
\caption{\textbf{T1} (COCO, ViT-B/16, frozen, $\lambda\!=\!1$, mean 3 seeds).
\textit{(a)}~$\Delta_{\rm CLIP}$ (hatched) is constant across all epochs and seeds (std\,$=0$), validating Assumption~\ref{ass:smooth}. It reverses the sign of $\Delta_{\rm LOSS}$ at $E_4{\to}E_5$, confirming both terms in Theorem~\ref{T1:loss-difference} are necessary.
\textit{(b)}~OLS $R^2$ rises from the frozen baseline ($\!\approx\!0$, dotted) to $+0.12$ (COCO) and $+0.16$ (CUB) at epoch~10, consistent with Theorem~\ref{T1:loss-difference}.}
\label{fig:fig_t1}
\end{figure*}

\noindent\textbf{T2 Experiments: Cross-Modal Contribution Balance Factor.} Theorem~\ref{thm2:BF-Gradient} predicts that the optimal $\mathbf{P}^*_{t\to t'}$ (Definition~\ref{def:CBF}) is achieved by scaling cross-modal similarity scores under Assumptions~\ref{ass:Sufficient class-normalized gradient drift}-\ref{ass:slow-gradient-magnitude}, and that $\lambda$ in $\mathcal{L}\!=\!\mathcal{L}_{CE}+\lambda\mathcal{L}_{CLIP}$ controls within-environment alignment $C_{LV}$,
that is computed by $C_{LV}(t):=\!\mathbb{E}_i[\langle z_{t,i},u_{t,i}\rangle]$.
We test both predictions.

\noindent\textbf{Observation~3: $\lambda$-sweep of $C_{LV}(t)$.} Under the frozen encoder, all metrics are exactly constant across $\lambda\!\in\!\{0.0,0.1,0.5,1.0,2.0,5.0\}$ on both datasets (supplementary), confirming encoder adaptation is required, consistent with the null result in E1.
Table~\ref{tab:clv-lambda} shows the adaptive results.

\begin{table}[h]
\caption{$C_{LV}(t)$ by $\lambda$, adaptive encoder, ViT-B/16 (mean, 3~seeds); $E_1$--$E_5$ as in Sec.~\ref{sec:setup}.
  Monotone increase confirms Theorem~\ref{thm2:BF-Gradient}. CUB saturates higher (${\approx}0.47$-$0.50$) than COCO (${\approx}0.27$-$0.36$).}
\label{tab:clv-lambda}\centering\small\setlength{\tabcolsep}{3pt}
\begin{tabular}{@{}r ccccc | ccccc@{}}
\toprule
& \multicolumn{5}{c|}{\textbf{COCO}} & \multicolumn{5}{c}{\textbf{CUB}} \\
$\lambda$ & $E_1$ & $E_2$ & $E_3$ & $E_4$ & $E_5$
          & $E_1$ & $E_2$ & $E_3$ & $E_4$ & $E_5$ \\
\midrule
0.0 & .048 & .057 & .036 & .030 & .032
    & .034 & .035 & .034 & .021 & .027 \\
0.1 & .280 & .284 & .195 & .156 & .169
    & .264 & .261 & .266 & .247 & .256 \\
0.5 & .341 & .327 & .224 & .230 & .243
    & .373 & .382 & .390 & .356 & .367 \\
1.0 & .344 & .342 & .229 & .258 & .271
    & .411 & .425 & .427 & .392 & .404 \\
2.0 & .358 & .348 & .229 & .270 & .294
    & .443 & .459 & .463 & .421 & .435 \\
5.0 & .357 & .355 & .231 & .276 & .332
    & .479 & .497 & .503 & .452 & .471 \\
\bottomrule
\end{tabular}
\end{table}

$C_{LV}(t)$ increases monotonically with $\lambda$ for all three architectures on both datasets, confirming Theorem~\ref{thm2:BF-Gradient}.
The large step from $\lambda\!=\!0$ to $\lambda\!=\!0.1$ (${\approx}0.04{\to}0.27$) shows that even minimal CLIP supervision immediately activates cross-modal alignment.

\noindent\textbf{Observation~4: CBF objective evaluation.} We evaluate both objectives from Definition~\ref{def:CBF}: $\mathrm{Obj1}\!:=\!\|\mathbf{P}\!\odot\!\nabla_\theta\mathcal{L} \|^2_F / \|\nabla_\theta\mathcal{L}\|^2_F$ and the score-based form $\mathrm{Obj2}$ from Theorem~\ref{thm2:BF-Gradient}, ~\eqref{thm:CBF-optimization}, at the post-training state $\theta^*_{t'}$.
Setting $\mathbf{P}\!=\!p\cdot\mathbf{1}$ with $p\!:=\!(Con_{\mathcal{I}_t\rightarrow\mathcal{T}_{t'}}+Con_{\mathcal{T}_t\rightarrow\mathcal{I}_{t'}})/2$, Definition~\ref{def:CBF} gives Obj1\,$=p^2$ analytically.

\begin{table}[h!]
\caption{CBF objectives at $\theta^*_{t'}$, CUB (representative, supplementary has full results).
  Obj1\,$=p^2$ (Definition~\ref{def:CBF}) to machine precision in all cells.
  Frozen: $\|\nabla_h\mathcal{L}_{CLIP}\|\!=\!0$.
  Obj2 is $\lambda$-invariant (null result, consistent with E3).
  Adaptive: $\|\nabla_\theta\mathcal{L}_{CLIP}\|\!>\!0$. 
  Obj2 range\,$>0$ for all 12 model-pair combinations, confirming Theorem~\ref{thm2:BF-Gradient}: encoder adaptation is required for $\mathbf{P}^*$ to be non-trivially $\lambda$-sensitive.}
\label{tab:cbf}\centering\small\setlength{\tabcolsep}{3pt}
\begin{tabular}{@{}llccc c@{}}
\toprule
\textbf{Enc.} & \textbf{Model} & \textbf{Pair}
  & \textbf{Obj1} & \textbf{Obj2 range} & $\lambda$\textbf{-dep.} \\
\midrule
\multirow{3}{*}{Frozen}
  & RN50 & E1${\to}$E2 & $p^2$ & 0.000 & no \\
  & ViT-B/32 & E2${\to}$E3 & $p^2$ & 0.000 & no \\
  & ViT-B/16 & E3${\to}$E4 & $p^2$ & 0.000 & no \\
\midrule
\multirow{3}{*}{Adaptive}
  & RN50 & E1${\to}$E2 & $p^2$ & 595 & yes \\
  & ViT-B/32 & E1${\to}$E2 & $p^2$ & 16.0 & yes \\
  & ViT-B/16 & E1${\to}$E2 & $p^2$ & 14.9 & yes \\
\bottomrule
\end{tabular}
\end{table}

Obj1\,$=p^2$ holds to machine precision, verifying the analytical identity of Definition~\ref{def:CBF}.
The frozen null result (Obj2 exactly $\lambda$-invariant) and the adaptive result ($\lambda$-dep.\ $=$\,yes for all 12 combinations) together confirm Theorem~\ref{thm2:BF-Gradient}: the optimal $\mathbf{P}^*$ requires encoder adaptation to be non-trivial.

\begin{figure*}[h!]
\centering
\includegraphics[width=\linewidth]{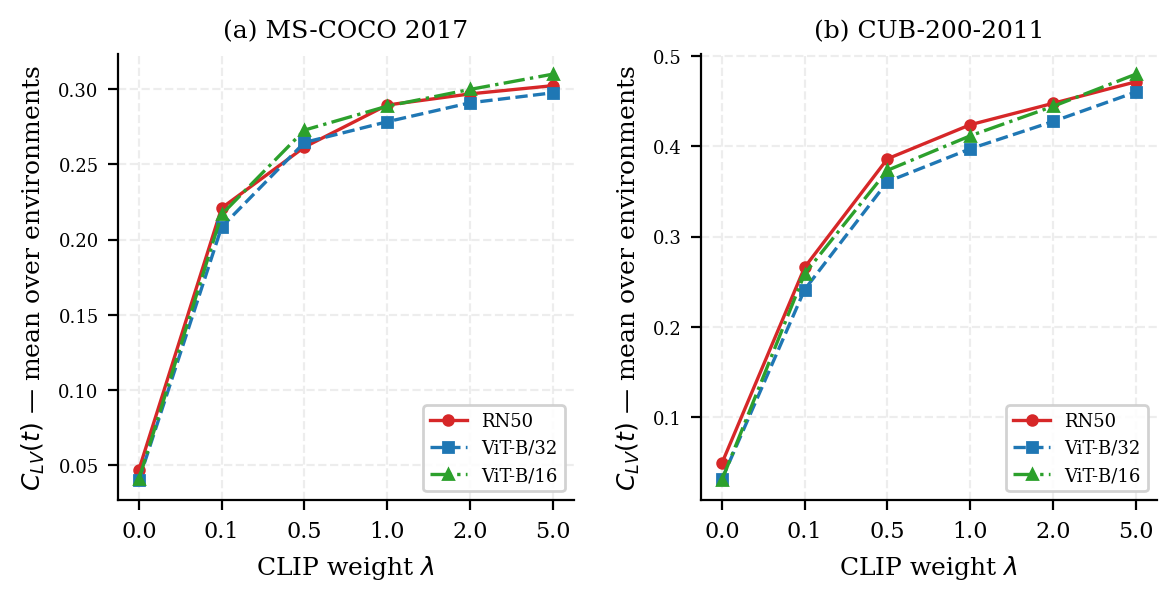}
\caption{\textbf{T2.}~Mean $C_{LV}(t)$ vs.\ $\lambda$, adaptive encoder, 3~models, 3~seeds, ep15: \textit{(a)}~COCO, \textit{(b)}~CUB.
All architectures show a monotone increase, confirming Theorem~\ref{thm2:BF-Gradient}. 
CUB saturates higher (${\approx}0.47$) than COCO (${\approx}0.30$). Under the frozen encoder $C_{LV}$ is constant $\forall\lambda$ (Appendix~\ref{app:frozen-lambda}).}
\label{fig:t2}
\end{figure*}
\noindent\textbf{T3 Experiments: Cross-Modal Contributions and Catastrophic Forgetting.} Theorem~\ref{thm3:Forgetting-Contribution} bounds forgetting $F_t\!:=\!\mathcal{L}_t(\theta^*_{t+1})-\mathcal{L}_t(\theta^*_t)$ under Assumption~\ref{ass:monotonic gradient of normalization terms}, predicting that orderings minimizing the cumulative contribution cost $\sum Con_{\mathcal{IT}}$ incur less total forgetting $F_\pi\!:=\!\sum_\ell F[\pi_\ell,\pi_{\ell+1}]$, where $F_t$ in (\ref{def:forgetting}) is computed as $F[t,t']\!:=\!\mathcal{L}_{CE}(\theta^*_{t'},\mathcal{D}_t) -\mathcal{L}_{CE}(\theta^*_t,\mathcal{D}_t)$ when $\mathcal{E}_t\to\mathcal{E}_{t'}$.
We test this over all 120 permutations (E5), then validate with full sequential training (E6).

\noindent\textbf{Observation~5: All-120-ordering sweep.} Table~\ref{tab:forgetting-matrix} shows the COCO pairwise forgetting matrix (ViT-B/16, frozen), used to compute $F_\pi$ for every ordering.

\begin{table}[h!]
\caption{$F_t$ (Theorem~\ref{thm3:Forgetting-Contribution}), COCO ViT-B/16, frozen (mean, 3~seeds).
  $E_2$ row (Animals, 85.7\% acc.) uniformly largest.
  $E_4$ column (Indoor, 32.0\%) uniformly smallest.}
\label{tab:forgetting-matrix}\centering\small\setlength{\tabcolsep}{3.5pt}
\begin{tabular}{@{}lrrrrr@{}}
\toprule
 & $h_1$ & $h_2$ & $h_3$ & $h_4$ & $h_5$ \\
\midrule
$E_1$ & $-$ & $+$1.53 & $+$1.31 & $+$1.26 & $+$1.23 \\
$E_2$ &$+$2.05& $-$ & $+$2.07 & $+$1.96 & $+$1.94 \\
$E_3$ &$+$1.30& $+$1.58 & $-$ & $+$1.13 & $+$1.15 \\
$E_4$ &$+$0.57& $+$0.81 & $+$0.27 & $-$ & $+$0.39 \\
$E_5$ &$+$0.53& $+$0.44 & $+$0.46 & $+$0.34 & $-$ \\
\bottomrule
\end{tabular}
\end{table}

\begin{table}[h!]
\caption{Ordering sweep over all 120 permutations, frozen encoder.
  Pearson: correlation between $F_\pi$ and $\sum Con_{\mathcal{IT}}$.
  Theorem~\ref{thm3:Forgetting-Contribution} predicts a negative value.
  COCO: prediction confirmed for all three architectures.
  CUB: $\sigma_I\!=\!0.063$ (vs $0.138$ COCO) compresses $Con_{\mathcal{IT}}$ to near-uniform values, collapsing the effective predictor range.
  The $1.16{\times}$ gap shows minimal ordering benefit is achievable in this regime.}
\label{tab:ordering-summary}\centering\small\setlength{\tabcolsep}{3pt}
\begin{tabular}{@{}ll rrrr@{}}
\toprule
 & \textbf{Model}
  & \textbf{Best $F_\pi$} & \textbf{Worst $F_\pi$}
  & \textbf{Ratio} & \textbf{Pearson} \\
\midrule
\multirow{3}{*}{COCO}
  & RN50 & 3.52 & 5.96 & $1.69{\times}$ & $-$0.082 \\
  & ViT-B/32 & 2.98 & 5.35 & $1.79{\times}$ & $-$0.138 \\
  & ViT-B/16 & 3.12 & 5.41 & $1.73{\times}$ & $-$0.125 \\
\midrule
\multirow{3}{*}{CUB}
  & RN50 & 1.35 & 1.56 & $1.16{\times}$ & $+$0.156 \\
  & ViT-B/32 & 1.04 & 1.20 & $1.16{\times}$ & $+$0.247 \\
  & ViT-B/16 & 1.11 & 1.29 & $1.17{\times}$ & $+$0.168 \\
\bottomrule
\end{tabular}
\end{table}

On COCO, the best ordering achieves up to $1.74{\times}$ lower $F_\pi$ than the worst, and Pearson $(F_\pi,\sum Con_{\mathcal{IT}})$ is negative for all three architectures, consistent with Theorem~\ref{thm3:Forgetting-Contribution}.
All best-20 orderings across all models place $E_2$ (Animals) last, structurally explained by the dominant $E_2$ row in Table~\ref{tab:forgetting-matrix}.
On CUB, $\sigma_I\!=\!0.063$ compresses all $Con_{\mathcal{IT}}$ to $[0.18,0.35]$, leaving $\sum Con_{\mathcal{IT}}$ near-uniform across orderings.
The positive Pearson and $1.16{\times}$ ratio reflect a noise-dominated regime, not a contradiction of Theorem~\ref{thm3:Forgetting-Contribution}.
Under the adaptive encoder on CUB, $F_\pi$ grows to 16-17 (${\approx}14{\times}$ over frozen) and Pearson remains positive.

\begin{figure*}[h!]
\centering
\includegraphics[width=\linewidth]{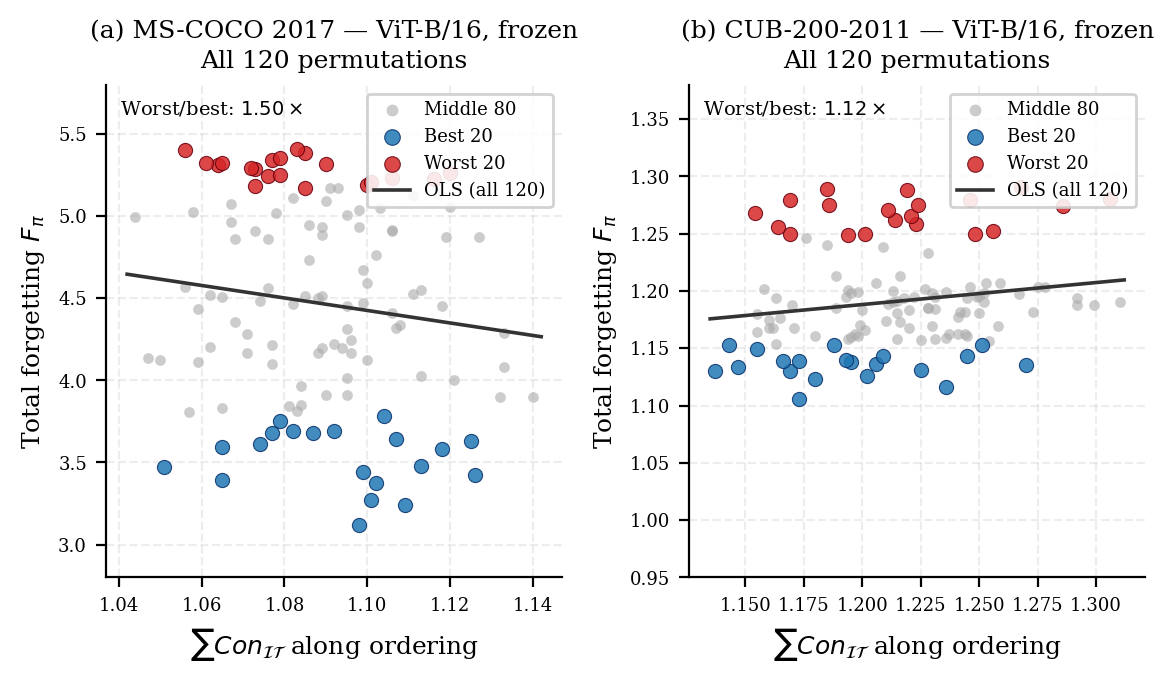}
\caption{\textbf{T3.}~$F_\pi$ vs.\ $\sum Con_{\mathcal{IT}}$, all 120 orderings (ViT-B/16, frozen); blue/red/grey: best-20/worst-20/middle-80.
\textit{(a)}~COCO: Pearson $r\!=\!-0.125$, worst/best $1.73{\times}$, confirming Theorem~\ref{thm3:Forgetting-Contribution}.
\textit{(b)}~CUB: $r\!=\!+0.17$, $1.17{\times}$, and near-uniform contributions ($\sigma_{\mathcal{I}}\!=\!0.063$ vs.\ $0.138$) collapse the predictor range.}
\label{fig:t3}
\end{figure*}


\noindent\textbf{Observation~6: Sequential continual learning.} We run full sequential training on the best, worst, and neutral orderings from E5, measuring backward transfer $\mathrm{BT}\!:=\!\sum_t[\mathrm{acc}_t(\text{after }t) -\mathrm{acc}_t(\text{final})]$.

\begin{table}[h!]
\caption{Backward transfer ($\mathrm{BT}$), frozen, ViT-B/16 (mean, 3~seeds).
  The 10\% reduction on COCO mirrors the $1.73{\times}$ $F_\pi$ gap in Observation~5, validating Theorem~\ref{thm3:Forgetting-Contribution} under full sequential training.
  Zero benefit on CUB is consistent with the near-uniform ordering effect observed in Observation~5.}
\label{tab:sequential}\centering\small\setlength{\tabcolsep}{3pt}
\begin{tabular}{@{}llrrr@{}}
\toprule
 & \textbf{Ordering}
  & \textbf{BT} & \textbf{Std} & \textbf{Acc.} \\
\midrule
\multirow{2}{*}{COCO}
  & neutral/worst & $+$2.011 & 0.020 & 0.169 \\
  & best & $+$1.806 & 0.019 & 0.210 \\
\midrule
\multirow{2}{*}{CUB}
  & neutral/worst & $+$2.931 & 0.008 & 0.169 \\
  & best & $+$2.931 & 0.021 & 0.169 \\
\bottomrule
\end{tabular}
\end{table}

The 10\% BT reduction on COCO directly mirrors the $1.73{\times}$ $F_\pi$ gap from Observation~5, confirming that the contribution-guided ordering identified through pairwise analysis translates faithfully to real sequential training.
On CUB, zero benefit is equally consistent: when $Con_{\mathcal{IT}}$ is near-uniform across environments, no ordering strategy can materially reduce forgetting, and the pairwise analysis in Observation~5 correctly predicts this limitation.
The low mean accuracy (${\approx}0.17$-$0.21$) is a known consequence of the single shared head, which retains only the final environment's decision boundary. 
The ordering benefit is nonetheless real within this constraint and arises solely from the contribution-score predictor with no additional supervision or architectural change.

\noindent\textbf{OOD Robustness (Empirical Investigation).}
At each environment $\mathcal{E}_t$, images from $\mathcal{E}_t$ are in-distribution (ID) and images from any $\mathcal{E}_{t'}\!\neq\!\mathcal{E}_t$ are OOD, exploiting the cross-environment distribution shifts as a natural testbed without dedicated OOD data.
MCM~\citep{ming2022delim} scores image $x$ as $S_t(x)\!=\!\max_{c}\cos(f(x),T_{t,c})/\tau$, and high scores indicate alignment with $\mathcal{E}_t$'s text prototypes (ID), while low scores indicate OOD.
We run two protocols: \textit{pairwise} (20 directed pairs $\mathcal{E}_t\!\to\!\mathcal{E}_{t'}$, scored with $\mathcal{E}_t$'s prototypes only) and \textit{sequential cumulative} (all seen prototypes pooled at each step $k$, measuring discriminability across the growing stream).
Full per-pair matrices, adaptive fine-tuning trajectories, and sequential results for all models are in Appendix~\ref{app:ood-protocol}.

\noindent\textbf{Observation~7: Pairwise OOD.}
Table~\ref{tab:ood} reports mean AUROC and Pearson$(AUROC, Con_{\mathcal{IT}})$ over 20 pairs (frozen encoder).
COCO achieves mean AUROC ${\approx}0.89$ with consistently negative Pearson ($-0.55$/$-0.40$/$-0.42$): lower cross-modal contribution predicts better OOD separation across all three architectures.
On CUB, near-uniform $Con_{\mathcal{IT}}$ ($\sigma_{\mathcal{I}}\!=\!0.063$) collapses the predictor and Pearson is positive, mirroring T3's boundary condition.
Under adaptive fine-tuning (ep15), mean AUROC drops (COCO $0.89{\to}0.61$; CUB $0.68{\to}0.50$) and CUB Pearson inverts to $-0.20$: encoder specialisation erodes zero-shot OOD robustness.

\noindent\textbf{Observation~8: Sequential cumulative OOD.}
The T3 contribution-guided ordering yields higher final cumulative AUROC on both datasets (Table~\ref{tab:ood}): COCO $0.757$ vs $0.715$ and CUB $0.506$ vs $0.488$ (neutral), confirming that the same ordering criterion simultaneously reduces forgetting (Table~\ref{tab:sequential}) and preserves OOD discriminability.

\begin{table}[H]
\caption{Pairwise MCM AUROC (frozen encoder, mean over 20 pairs, 3~seeds) and
  Pearson$(AUROC, Con_{\mathcal{IT}})$; lower block: final-step cumulative AUROC (ViT-B/16).
  COCO: negative Pearson (lower $Con_{\mathcal{IT}}$ $\Rightarrow$ better separation).
  CUB: positive (near-uniform regime, same boundary condition as T3).}
\label{tab:ood}\centering\small\setlength{\tabcolsep}{3.5pt}
\begin{tabular}{@{}l cc cc@{}}
\toprule
& \multicolumn{2}{c}{\textbf{COCO (frozen)}} & \multicolumn{2}{c}{\textbf{CUB (frozen)}} \\
\textbf{Model} & AUROC & Pearson & AUROC & Pearson \\
\midrule
RN50 & 0.892 & $-0.42$ & 0.654 & $+0.29$ \\
ViT-B/32 & 0.890 & $-0.40$ & 0.682 & $+0.07$ \\
ViT-B/16 & 0.897 & $-0.55$ & 0.700 & $+0.07$ \\
\bottomrule
\end{tabular}
\end{table}
\section{Conclusion}
\label{sec:conclusion}

We introduced modality-specific contribution scores as a formal lens for catastrophic forgetting in continual VLMs, proving three connected results.
Loss drift between any two environments decomposes linearly into the cross-modal contributions for both loss components, with the CLIP term carrying exactly zero variance under a frozen encoder and being necessary to predict inter-environment transition direction~(T1).
The Contribution Balance Factor reduces analytically to $p^2$ under a contribution-normalized mask, confirmed to machine precision across all settings, and the CLIP weight $\lambda$ monotonically controls within-environment alignment $C_{LV}(t)$ under adaptive fine-tuning, while the CLIP gradient is exactly zero under a frozen encoder~(T2).
A closed-form forgetting bound predicts that ordering environments by cumulative contribution cost reduces forgetting; on MS-COCO, where cross-modal contributions vary substantially, this yields up to $1.74{\times}$ reduction across all 120 permutations and three backbones, with Pearson correlation consistently negative~(T3).
On CUB-200-2011, where contributions are near-uniform ($\sigma_I\!=\!0.063$ vs $0.138$ on COCO), the ordering range narrows to $1.16{\times}$ and the Pearson direction is not preserved, showing that the prediction's discriminative power is governed by the degree of cross-modal variation in the embedding space.
Empirically, the same contribution scores predict OOD discriminability on COCO and the contribution-guided ordering simultaneously reduces forgetting and improves cumulative OOD AUROC, suggesting a unified role for these scores across both objectives.
These results suggest that contribution scores could serve directly as a training signal for principled continual VLM algorithms, and motivate future work in settings where environment-level cross-modal variation is large.

\section{Acknowledgments:}
This work has been partially supported by NSF CAREER CCF-2451457. The findings are those of the authors only and do not represent any position of these funding bodies.

{\small
\bibliographystyle{ieeenat_fullname}
\bibliography{egbib}

@article{zhou2025learning,
  title={Learning without forgetting for vision-language models},
  author={Zhou, Da-Wei and Zhang, Yuanhan and Wang, Yan and Ning, Jingyi and Ye, Han-Jia and Zhan, De-Chuan and Liu, Ziwei},
  journal={IEEE Transactions on Pattern Analysis and Machine Intelligence},
  year={2025},
  publisher={IEEE}
}

@article{wang2024comprehensive,
  title={A comprehensive survey of continual learning: Theory, method and application},
  author={Wang, Liyuan and Zhang, Xingxing and Su, Hang and Zhu, Jun},
  journal={IEEE transactions on pattern analysis and machine intelligence},
  volume={46},
  number={8},
  pages={5362--5383},
  year={2024},
  publisher={IEEE}
}

@article{mirzadeh2020understanding,
  title={Understanding the role of training regimes in continual learning},
  author={Mirzadeh, Seyed Iman and Farajtabar, Mehrdad and Pascanu, Razvan and Ghasemzadeh, Hassan},
  journal={Advances in neural information processing systems},
  volume={33},
  pages={7308--7320},
  year={2020}
}

@InProceedings{radford2021clip,
  author = "A. Radford and J.~W. Kim and C. Hallacy and A. Ramesh and G. Goh and S. Agarwal and G. Sastry and A. Askell and P. Mishkin and J. Clark and G. Krueger and I. Sutskever",
  title = "Learning Transferable Visual Models From Natural Language Supervision",
  booktitle = "Proc. ICML",
  year = "2021",
}

@InProceedings{fort2021clip,
  author = "S. Fort and J. Ren and B. Lakshminarayanan",
  title = "Exploring the Limits of Out-of-Distribution Detection",
  booktitle = "Proc. NeurIPS",
  year = "2021",
}

@InProceedings{ming2022delim,
  author = "Y. Ming and Z. Cai and J. Gu and Y. Sun and W. Li and Y. Li",
  title = "Delving into Out-of-Distribution Detection with Vision-Language Representations",
  booktitle = "Proc. NeurIPS",
  year = "2022",
}

@InProceedings{wang2023clipn,
  author = "H. Wang and Z. Li and L. Feng and W. Zhang",
  title = "CLIPN for Zero-Shot Out-of-Distribution Detection: Teaching CLIP to Say No",
  booktitle = "Proc. ICCV",
  year = "2023",
}

@InProceedings{esmaeilpour2022zero,
  author = "S. Esmaeilpour and B. Liu and E. Robertson and L. Shu",
  title = "Zero-Shot Out-of-Distribution Detection Based on the Pre-trained Model CLIP",
  booktitle = "Proc. AAAI",
  year = "2022",
}

@article{kirkpatrick2017overcoming,
  author = "J. Kirkpatrick and R. Pascanu and N. Rabinowitz and J. Veness and G. Desjardins and A.~A. Rusu and K. Milan and J. Quan and T. Ramalho and A. Grabska-Barwinska and D. Hassabis and C. Clopath and D. Kumaran and R. Hadsell",
  title = "Overcoming Catastrophic Forgetting in Neural Networks",
  journal = "PNAS",
  volume = "114",
  number = "13",
  pages = "3521-3526",
  year = "2017",
}

@InProceedings{zenke2017si,
  author = "F. Zenke and B. Poole and S. Ganguli",
  title = "Continual Learning Through Synaptic Intelligence",
  booktitle = "Proc. ICML",
  year = "2017",
}

@InProceedings{aljundi2018mas,
  author = "R. Aljundi and F. Babiloni and M. Elhoseiny and M. Rohrbach and T. Tuytelaars",
  title = "Memory Aware Synapses: Learning What (Not) to Forget",
  booktitle = "Proc. ECCV",
  year = "2018",
}

@InProceedings{lopez2017gradient,
  author = "D. Lopez-Paz and M. Ranzato",
  title = "Gradient Episodic Memory for Continual Learning",
  booktitle = "Proc. NeurIPS",
  year = "2017",
}

@InProceedings{farajtabar2020orthogonal,
  author = "M. Farajtabar and N. Azizan and A. Mott and A. Li",
  title = "Orthogonal Gradient Descent for Continual Learning",
  booktitle = "Proc. AISTATS",
  year = "2020",
}

@InProceedings{saha2021gpm,
  author = "G. Saha and I. Garg and K. Roy",
  title = "Gradient Projection Memory for Continual Learning",
  booktitle = "Proc. ICLR",
  year = "2021",
}

@InProceedings{rebuffi2017icarl,
  author = "S.-A. Rebuffi and A. Kolesnikov and G. Sperl and C.~H. Lampert",
  title = "iCaRL: Incremental Classifier and Representation Learning",
  booktitle = "Proc. CVPR",
  year = "2017",
}

@InProceedings{rolnick2019er,
  author = "D. Rolnick and A. Ahuja and J. Schwarz and T. Lillicrap and G. Wayne",
  title = "Experience Replay for Continual Learning",
  booktitle = "Proc. NeurIPS",
  year = "2019",
}

@InProceedings{shin2017dgr,
  author = "H. Shin and J.~K. Lee and J. Kim and J. Kim",
  title = "Continual Learning with Deep Generative Replay",
  booktitle = "Proc. NeurIPS",
  year = "2017",
}

@misc{rusu2016pnn,
  author = "A.~A. Rusu and N.~C. Rabinowitz and G. Desjardins and H. Soyer and J. Kirkpatrick and K. Kavukcuoglu and R. Pascanu and R. Hadsell",
  title = "Progressive Neural Networks",
  note = "arXiv:1606.04671",
  year = "2016",
}

@InProceedings{mallya2018packnet,
  author = "A. Mallya and S. Lazebnik",
  title = "PackNet: Adding Multiple Tasks to a Single Network by Iterative Pruning",
  booktitle = "Proc. CVPR",
  year = "2018",
}

@InProceedings{wang2022l2p,
  author = "Z. Wang and Z. Zhang and C.-Y. Lee and H. Zhang and R. Sun and X. Ren and G. Su and V. Perot and J. Dy and T. Pfister",
  title = "Learning to Prompt for Continual Learning",
  booktitle = "Proc. CVPR",
  year = "2022",
}

@InProceedings{wang2022dualprompt,
  author = "Z. Wang and Z. Zhang and H. Zhang and C.-Y. Lee and R. Sun and X. Ren and G. Su and V. Perot and J. Dy and T. Pfister",
  title = "DualPrompt: Complementary Prompting for Rehearsal-Free Continual Learning",
  booktitle = "Proc. ECCV",
  year = "2022",
}

@InProceedings{zhou2024ease,
  author = "D.-W. Zhou and H.-J. Ye and L. Zhan and Z.-H. Zhou",
  title = "EASE: Exemplar-Free Class Incremental Learning with Structure Expansion",
  booktitle = "Proc. CVPR",
  year = "2024",
}

@InProceedings{zheng2023zscl,
  author = "Z. Zheng and M. Ma and K. Wang and Z. Qin and X. Yue and Y. You",
  title = "Preventing Zero-Shot Transfer Degradation in Continual Learning of Vision-Language Models",
  booktitle = "Proc. ICCV",
  year = "2023",
}

@misc{naiknaware2024tempscone,
  author = "A. Naiknaware and S.~Y. Sekeh",
  title = "Temp-SCONE: A Novel Out-of-Distribution Detection and Domain Generalization Framework for Wild Data with Temporal Shift",
  note = "arXiv:2512.04571",
  year = "2024",
}

@misc{naiknaware2025tqpm,
  author = "A. Naiknaware and S.~Y. Sekeh",
  title  = "T-QPM: Enabling Temporal Out-of-Distribution Detection and Domain Generalization for Vision-Language Models in Open-World",
  note = "arXiv:2603.18481",
  year = "2025",
}

@InProceedings{lin2014coco,
  author = "T.-Y.~Lin and M.~Maire and S.~Belongie and J.~Hays and P.~Perona and D.~Ramanan and P.~Doll{\'a}r and C.~L.~Zitnick",
  title = "Microsoft COCO: Common Objects in Context",
  booktitle = "Proc. ECCV",
  year = "2014",
}

@misc{wah2011cub,
  author = "C.~Wah and S.~Branson and P.~Welinder and P.~Perona and S.~Belongie",
  title = "The Caltech-UCSD Birds-200-2011 Dataset",
  note = "California Institute of Technology, CNS-TR-2011-001",
  year = "2011",
}

@InProceedings{reed2016cub,
  author = "S.~Reed and Z.~Akata and H.~Lee and B.~Schiele",
  title = "Learning Deep Representations of Fine-Grained Visual Descriptions",
  booktitle = "Proc. CVPR",
  year = "2016",
}

@InProceedings{dosovitskiy2021vit,
  author = "A.~Dosovitskiy and L.~Beyer and A.~Kolesnikov and D.~Weissenborn and X.~Zhai and T.~Unterthiner and M.~Dehghani and M.~Minderer and G.~Heigold and S.~Gelly and J.~Uszkoreit and N.~Houlsby",
  title = "An Image is Worth 16x16 Words: Transformers for Image Recognition at Scale",
  booktitle = "Proc. ICLR",
  year = "2021",
}

@InProceedings{he2016resnet,
  author = "K.~He and X.~Zhang and S.~Ren and J.~Sun",
  title = "Deep Residual Learning for Image Recognition",
  booktitle = "Proc. CVPR",
  year = "2016",
}

@InProceedings{cherti2023openclip,
  author = "M.~Cherti and R.~Beaumont and R.~Wightman and M.~Wortsman and G.~Ilharco and C.~Gordon and C.~Schuhmann and L.~Schmidt and J.~Jitsev",
  title = "Reproducible Scaling Laws for Contrastive Language-Image Learning",
  booktitle = "Proc. CVPR",
  year = "2023",
}

@InProceedings{wolf2020hf,
  author = "T.~Wolf and L.~Debut and V.~Sanh and J.~Chaumond and C.~Delangue and A.~Moi and P.~Cistac and T.~Rault and R.~Louf and M.~Funtowicz and J.~Brew",
  title = "Transformers: State-of-the-Art Natural Language Processing",
  booktitle = "Proc. EMNLP (Systems Demonstrations)",
  year = "2020",
}

@article{gao2021clipadapter,
  author = "P.~Gao and S.~Geng and R.~Zhang and T.~Ma and R.~Chen and Z.~Li and H.~Qiu and H.~Li",
  title = "CLIP-Adapter: Better Vision-Language Models with Feature Adapters",
  journal = "IJCV",
  volume = "132",
  pages = "581-595",
  year = "2024",
}

@article{zhou2022coop,
  author = "K.~Zhou and J.~Yang and C.~C.~Loy and Z.~Liu",
  title = "Learning to Prompt for Vision-Language Models",
  journal = "IJCV",
  volume = "130",
  number = "9",
  pages = "2337-2348",
  year = "2022",
}

@InProceedings{zhou2022cocoop,
  author = "Kaiyang Zhou and Jingkang Yang and Chen Change Loy and Ziwei Liu",
  title = "Conditional Prompt Learning for Vision-Language Models",
  booktitle = "Proc. CVPR",
  pages = "16816-16825",
  year = "2022",
}

@article{marouf2025quad,
  author = "Imad Eddine Marouf and Enzo Tartaglione and St'ephane Lathuilière and Joost van de Weijer",
  title = "Ask and Remember: Questions-Only Replay Strategy for Continual Visual Question Answering",
  journal = "arXiv preprint arXiv:2502.04469",
  year = "2025",
}

@misc{openai2021cliprn50,
  author = "OpenAI",
  title = "CLIP Model Card",
  note = "OpenAI GitHub, model: RN50-quickgelu",
  year = "2021",
}

@misc{openai2021clipvitb32,
  author = "OpenAI",
  title = "CLIP ViT-B/32 Model Card",
  note = "HuggingFace Hub: openai/clip-vit-base-patch32",
  year = "2021",
}

@misc{openai2021clipvitb16,
  author = "OpenAI",
  title = "CLIP ViT-B/16 Model Card",
  note = "HuggingFace Hub: openai/clip-vit-base-patch16",
  year = "2021",
}
}

\clearpage
\setcounter{page}{1}

\twocolumn[{%
\centering
\Large\bfseries Supplementary Material\par
\vspace{1.5em}
}]

\section{Extended Related Work}\label{app:related_work}

\noindent\textbf{Continual Learning and Catastrophic Forgetting.}
Continual learning addresses the stability-plasticity dilemma: how to acquire new knowledge without overwriting what was learned before.
Regularization-based methods penalize updates to parameters identified as important for prior tasks.
EWC~\cite{kirkpatrick2017overcoming} uses the Fisher information matrix as an importance measure.
SI~\cite{zenke2017si} accumulates an online importance estimate during training proportional to each parameter's contribution to the loss reduction, and MAS~\cite{aljundi2018mas} estimates importance from the sensitivity of the output to input perturbations, requiring no class labels.
Gradient-based methods constrain the learning direction to avoid interference: GEM~\cite{lopez2017gradient} enforces that parameter updates do not increase the loss on stored episodic memory, OGD~\cite{farajtabar2020orthogonal} projects gradients to be orthogonal to those of earlier tasks, and GPM~\cite{saha2021gpm} maintains a basis spanning the gradient subspace of prior tasks and restricts new updates to its null space.
Replay-based methods rehearse earlier tasks by storing exemplars~\cite{rebuffi2017icarl}, uniform random buffers~\cite{rolnick2019er}, or generative approximations~\cite{shin2017dgr} alongside new training.
Architecture-based approaches avoid interference by allocating separate capacity per task through progressive networks~\cite{rusu2016pnn} or iterative weight masking~\cite{mallya2018packnet}.
These methods measure forgetting exclusively through classification accuracy.
Our work is structurally related to this literature but introduces modality-specific alignment scores as a richer theoretical lens for characterizing forgetting, providing closed-form bounds that go beyond accuracy-based metrics.

\noindent\textbf{Continual Learning for Vision-Language Models.}
Large vision-language models such as CLIP~\cite{radford2021clip} provide a rich frozen embedding space that enables new continual learning strategies.
Because the pre-trained encoder already encodes diverse semantic concepts, methods that update only lightweight modules around the frozen backbone can achieve strong continual performance without the catastrophic forgetting that full fine-tuning produces.
Prompt-based methods adapt CLIP by learning small sets of context tokens prepended to the input.
CoOp~\cite{zhou2022coop} optimizes a single shared context vector across all classes, and CoCoOp~\cite{zhou2022cocoop} conditions the context on the input image to improve generalization to unseen classes.
L2P~\cite{wang2022l2p} extends prompt tuning to the continual setting by maintaining a pool of prompts and dynamically retrieving the most relevant subset for each input, while DualPrompt~\cite{wang2022dualprompt} separates general task-agnostic prompts from task-specific ones to balance stability and plasticity.
EASE~\cite{zhou2024ease} enables exemplar-free class-incremental learning by expanding the prompt subspace with each new task, avoiding the need to store any past data.
Adapter-based methods insert lightweight residual modules within the encoder: CLIP-Adapter~\cite{gao2021clipadapter} adds a small feature adapter that blends adapted and original CLIP features at inference time, achieving competitive performance with only a fraction of the parameters that full fine-tuning requires.
ZSCL~\cite{zheng2023zscl} takes a different approach, preventing zero-shot transfer degradation during sequential fine-tuning by interpolating the current model weights with those of the original CLIP checkpoint after each task update.
This simple weight-space interpolation is remarkably effective at maintaining the pre-trained alignment while still allowing task adaptation.
Multi-modal continual learning extends these ideas beyond unimodal classification.
VQA-CL~\cite{marouf2025quad} studies catastrophic forgetting in visual question answering, where both the visual and linguistic representations must remain coherent across tasks.
This finding directly motivates our theoretical framework: a formal account of how cross-modal alignment evolves across environments is essential for understanding and bounding forgetting.

\noindent\textbf{OOD Detection and Continual OOD.}
Alignment-based detectors exploit VLM joint embeddings: MCM~\cite{ming2022delim} scores images against class-level text prototypes, CLIPN~\cite{wang2023clipn} teaches CLIP to produce explicit rejection responses, and zero-shot variants~\cite{esmaeilpour2022zero,fort2021clip} require no task-specific training.
At the intersection with continual learning, Temp-SCONE~\cite{naiknaware2024tempscone} and T-QPM~\cite{naiknaware2025tqpm} address temporal distribution shift through temporal regularization and adaptive fusion, respectively.
We use MCM as the detector in our empirical investigation and examine whether the contribution scores that govern forgetting also relate to OOD discriminability across sequential environments.

\section{Proofs of Section~\ref{sec:theory}}\label{app:proofs}
n this section, we provide detailed proofs of the main theoretical results in Section~\ref{sec:theory}. 
Throughout, we use the same notation as in the main text. In particular, the image representations for the environment $\mathcal{E}_t$
are $z_{t,i} = f(I_{t,i}) \in \mathbb{R}^{d\times 1}$ and the text representations are $u_{t,i} = g(T_{t,i}) \in \mathbb{R}^{d\times 1}$.
\subsection{Proofs of T1 Analysis}
In this section, we assume that sample size in the environments are equal i.e. $n_t=n_{t'}=n$. 

\subsubsection{Proof of Lemma~\ref{lem.1}}
\label{app:proof_lem_1}

\begin{proof}
Recall the $\Delta_{CE}(t\rightarrow t')=\mathcal{L}_{CE}(\theta_t)-\mathcal{L}_{CE}(\theta_{t'})$ and develop the following formulation using cross-entropy loss, $\mathcal{L}_{CE}$ in (\ref{Loss-function(ID)}).  
\begin{multline}\label{eq:lem-1-1}
\Delta_{CE}(t\rightarrow t')\\
=-\frac{1}{n}\sum_{i=1}^n\Bigg[\log\frac{e^{z^{id}_{t,i}\cdot u^{id}_{t,y_i}/\tau}}{\sum_{c=1}^C e^{z^{id}_{t,i}\cdot u^{id}_{t,c}/\tau}}\\
-\log\frac{e^{\tilde{z}^{id}_{t',i}\cdot \tilde{u}^{id}_{t',\tilde{y}_i}/\tau}}{\sum_{c=1}^C e^{\tilde{z}^{id}_{t',i}\cdot \tilde{u}^{id}_{t',c}/\tau}}\Bigg]\\
=-\frac{1}{n}\sum_{i=1}^n\Bigg[{z^{id}_{t,i}\cdot u^{id}_{t,y_i}/\tau} -{\tilde{z}^{id}_{t',i}\cdot \tilde{u}^{id}_{t',\tilde{y}_i}/\tau}\\
-\log\Big(\sum_{c=1}^C e^{z^{id}_{t,i}\cdot u^{id}_{t,c}/\tau}\Big)+\log\Big(\sum_{c=1}^C e^{\tilde{z}^{id}_{t',i}\cdot \tilde{u}^{id}_{t',c}/\tau}\Big)\Bigg]. 
\end{multline}
Set the normalizations terms by
\begin{align*}
\alpha^{C}_{t,i}&:=\sum_{c=1}^C e^{z^{id}_{t,i}\cdot u^{id}_{t,c}/\tau},\\
\alpha^{C}_{t',i}&:=\sum_{c=1}^C e^{\tilde{z}^{id}_{t',i}\cdot \tilde{u}^{id}_{t',c}/\tau},
\end{align*}
this simplifies (\ref{eq:lem-1-1}) as  
\begin{multline}\label{Delta-CE}
\Delta_{CE}(t\rightarrow t') = -\frac{1}{2n\tau}\sum_{i=1}^n \Bigl[2z^{id}_{t,i}\cdot u^{id}_{t,y_i} - 2\tilde{z}^{id}_{t',i}\cdot \tilde{u}^{id}_{t',\tilde{y}_i}\\
+2\tau\log(\alpha^C_{t',i}/\alpha^C_{t,i})\Bigr]
\end{multline}
Next, we focus on the internal term in (\ref{Delta-CE}) and add-subtract the term $z^{id}_{t.i}\tilde{u}^{id}_{t',\tilde{y}_i}$:
\begin{multline}\label{eq:lem1-3}
z^{id}_{t,i}\cdot u^{id}_{t,y_i}-\tilde{z}^{id}_{t',i}\cdot \tilde{u}^{id}_{t',\tilde{y}_i}\\
=z^{id}_{t,i}\cdot u^{id}_{t,y_i}-\tilde{z}^{id}_{t',i}\cdot \tilde{u}^{id}_{t',\tilde{y}_i}\\
\quad-z^{id}_{t.i}\tilde{u}^{id}_{t',\tilde{y}_i}+z^{id}_{t.i}\tilde{u}^{id}_{t',\tilde{y}_i}\\
=z^{id}_{t,i}\cdot(u^{id}_{t,y_i}-\tilde{u}^{id}_{t',\tilde{y}_i})\\
\quad+\tilde{u}^{id}_{t',\tilde{y}_i}\cdot(z^{id}_{t,i}-\tilde{z}^{id}_{t',i}),
\end{multline}
Using Assumption~\ref{ass:smooth}, $\exists \beta_{u,i}, \beta_{z,i}$ such that $u^{id}_{t,y_i}-\tilde{u}^{id}_{t',\tilde{y}_i}=\beta_{u,i}\;\tilde{u}^{id}_{t',\tilde{y}_i}$ and $z^{id}_{t,i}-\tilde{z}^{id}_{t',i}=-\beta_{z,i}\;{z}^{id}_{t,i}$. This implies that (\ref{eq:lem1-3}) equals to 
\begin{multline}\label{eq:lem1-2}
\beta_{u,i}(z^{id}_{t,i}\cdot \tilde{u}^{id}_{t',\tilde{y}_i})-\beta_{z,i}(\tilde{u}^{id}_{t',\tilde{y}_i}\cdot z^{id}_{t,i})\\
=\bar{\beta}_{u,i}cos(z^{id}_{t,i},\tilde{u}^{id}_{t',\tilde{y}_i})-\bar{\beta}_{z,i}cos(\tilde{u}^{id}_{t',\tilde{y}_i},z^{id}_{t,i})\\
=cos(\tilde{u}^{id}_{t',\tilde{y}_i},z^{id}_{t,i})[\bar{\beta}_{u,i}-\bar{\beta}_{z,i}]\\
=\bar{\kappa}_{1,i}\; cos(\tilde{u}^{id}_{t',\tilde{y}_i},z^{id}_{t,i}).
\end{multline}
The last equality is driven because cosine similarity is symmetric and 
\begin{align*}
\bar{\beta}_{u,i}&={\beta}_{u,i}\|z^{id}_{t,i}\|\|\tilde{u}^{id}_{t',\tilde{y}_i}\|,\\
\bar{\beta}_{z,i}&={\beta}_{z,i}\|\tilde{u}^{id}_{t',\tilde{y}_i}\| \|z^{id}_{t,i}\|. 
\end{align*}
In (\ref{eq:lem1-2}), $\bar{\beta}_{u,i}-\bar{\beta}_{z,i}=\bar{\kappa}_{1,i}$. Similarly
\begin{multline}\label{eq:lem1-4}
z^{id}_{t,i}\cdot u^{id}_{t,y_i}-\tilde{z}^{id}_{t',i}\cdot \tilde{u}^{id}_{t',\tilde{y}_i} =\\
z^{id}_{t,i}\cdot u^{id}_{t,y_i}-\tilde{z}^{id}_{t',i}\cdot \tilde{u}^{id}_{t',\tilde{y}_i} -\tilde{z}^{id}_{t.i}{u}^{id}_{t,{y}_i}+\tilde{z}^{id}_{t.i}{u}^{id}_{t,{y}_i}\\
=(z^{id}_{t,i}-\tilde{z}^{id}_{t',i})\cdot u^{id}_{t,y_i} +\tilde{z}^{id}_{t',i}\cdot(u^{id}_{t,y_i}-\tilde{u}^{id}_{t',\tilde{y}_i})\\
=\bar{\kappa}_{2,i} \; cos(u^{id}_{t,y_i},\tilde{z}^{id}_{t',i}),\\
\hbox{where $\bar{\kappa}_{2,i}$ is a function of $\beta_{u,i}$ and $\beta_{z,i}$.} 
\end{multline}
Plugging (\ref{eq:lem1-2}) and (\ref{eq:lem1-4}) into (\ref{Delta-CE}) yields
\begin{multline}
\Delta_{CE}(t\rightarrow t')=-\frac{1}{2n\;\tau}\sum_{i=1}^n\Bigg[\bar{\kappa}_{1,i}\; cos(\tilde{u}^{id}_{t',\tilde{y}_i},z^{id}_{t,i})\\
+\bar{\kappa}_{2,i} \; cos(u^{id}_{t,y_i},\tilde{z}^{id}_{t',i})+2\tau\log(\alpha^{C}_{t',i}/\alpha^{C}_{t,i})\Bigg].
\end{multline}
Recalling cross-modality contributions in Def.~\ref{def:Cross-Modality Contribution}, we prove the formula (\ref{main:lem-1}). Note that without loss of generality, we assume that for all samples $\bar{\kappa}_{1,i}=\bar{\kappa}_{1}$ and $\bar{\kappa}_{2,i}=\bar{\kappa}_{2}$. 
\end{proof}

\subsubsection{Proof of Lemma~\ref{lem.2}}
\label{app:proof_lem_2}

\begin{proof}
First simplify $\mathcal{L}_{CLIP}$ in (\ref{Loss-function(CLIP)}) by leveraging similar arguments as in Lemma~\ref{lem.1}:
\begin{multline}\label{eq:lem-2-1}
\mathcal{L}_{CLIP}(\theta_t)=\frac{1}{2n_t}\sum_{i=1}^{n_t}\Big[-z^{id}_{t,i}\cdot u^{id}_{t,i}/\tau\\
-u^{id}_{t,i}\cdot z^{id}_{t,i}/\tau+\log \alpha_{t,i}^{z}+\log \alpha_{t,i}^{u}\Big],
\end{multline}
where $\alpha_{t,i}^{z},\alpha_{t,i}^{u}$ are normalization terms:
\begin{align*}
\alpha_{t,i}^{z}&:= \sum_{j=1}^{n_t} e^{z^{id}_{t,i}\cdot u^{id}_{t,j}/\tau },\\
\alpha_{t,i}^{u}&:= \sum_{j=1}^{n_t} e^{u^{id}_{t,i}\cdot z^{id}_{t,j}/\tau }.
\end{align*}
Next, we use (\ref{eq:lem-2-1}) and write the loss different $\Delta_{CLIP}(t\rightarrow t')$ for $\mathcal{E}_t$ and $\mathcal{E}_{t'}$ by
\begin{align}\label{eq:lem2-3}
\Delta_{CLIP}(t\rightarrow t')=\\
\frac{1}{2n\tau}\sum_{i=1}^{n}\Big[-z^{id}_{t,i}\cdot u^{id}_{t,i}-u^{id}_{t,i}\cdot z^{id}_{t,i}\nonumber\\
+\tilde{z}^{id}_{t',i}\cdot \tilde{u}^{id}_{t',i}+\tilde{u}^{id}_{t',i}\cdot \tilde{z}^{id}_{t',i}\Big]\nonumber\\
+\frac{1}{2n}\sum_{i=1}^n\Big[\tau\log \alpha^{z}_{t,i}/\tilde{\alpha}^{\tilde{z}}_{t',i}\nonumber+\tau\log \alpha^{u}_{t,i}/\tilde{\alpha}^{\tilde{u}}_{t',i}\Big]
\end{align}
Going back to Lemma~\ref{lem.1} and $\Delta_{CE}(t\rightarrow t')$ proof arguments, under Assumption~\ref{ass:smooth}, we write
\begin{align}\label{eq:lem2-2}
\tilde{z}^{id}_{t',i}\cdot \tilde{u}^{id}_{t',i}-z^{id}_{t,i}\cdot u^{id}_{t,i}&=\kappa_{1,i}\;cos(z^{id}_{t,i},\tilde{u}^{id}_{t',i}),\nonumber\\
\tilde{z}^{id}_{t',i}\cdot \tilde{u}^{id}_{t',i}-z^{id}_{t,i}\cdot u^{id}_{t,i}&=\kappa_{2,i}\;cos(u^{id}_{t,i},\tilde{z}^{id}_{t',i}), 
\end{align}
Which by plugging (\ref{eq:lem2-2}) into (\ref{eq:lem2-3}) and assuming that $\kappa_{1,i}=\kappa_{1}$, $\kappa_{2,i}=\kappa_{2}$, we prove (\ref{main:lem-2}). 
\end{proof}

\subsubsection{Proof of Theorem~\ref{T1:loss-difference}}
\label{app:disscuss_thm_1}

Going back to total loss formulation in (\ref{total-loss}), we have the total loss difference as
\begin{align}\label{eq:thm-1-1}
\Delta_{Loss}(t\rightarrow t')=\Delta_{CE}(t\rightarrow t')+\lambda \;\Delta_{CLIP}(t\rightarrow t'). 
\end{align}
Note that we assumed $\lambda_t=\lambda$ without loss of generality. Now use formula (\ref{main:lem-1}) in Lemma~\ref{lem.1} and formula (\ref{main:lem-2}) in Lemma~\ref{lem.2} and plug them into (\ref{eq:thm-1-1}):
\begin{multline}\label{eq:thm2-3}
\Delta_{Loss}(t\rightarrow t')=\frac{1}{2\tau}\Big[\beta_1Con_{\mathcal{I}_t\rightarrow\mathcal{T}_{t'}}\\
+\beta_2 Con_{\mathcal{T}_t\rightarrow\mathcal{I}_{t'}}+\frac{\tau}{n}\sum_{i=1}^n\alpha_i\Big], 
\end{multline}
where $\beta_1=\lambda\kappa_1-\bar{\kappa}_1$, $\beta_1=\lambda\kappa_2-\bar{\kappa}_2$, and 
\begin{equation}
\alpha_i=2\log(\alpha^C_{t',i}/\alpha^C_{t,i})+\lambda \left(\log \alpha^{z}_{t,i}/\tilde{\alpha}^{\tilde{z}}_{t',i}+\log \alpha^{u}_{t,i}/\tilde{\alpha}^{\tilde{u}}_{t',i}\right).
\end{equation}
The formula (\ref{eq:thm2-3}) proves our claim in Theorem~\ref{T1:loss-difference} i.e $\Delta_{loss}$ is a linear regression of $Con_{\mathcal{I}_t\rightarrow\mathcal{T}_{t'}}$ and $Con_{\mathcal{T}_t\rightarrow\mathcal{I}_{t'}}$, with regression multipliers $\beta_1/2\tau, \beta_2/2\tau$ and offset $\frac{1}{2n}\sum_{i=1}^n\alpha_i$. 

\subsection{Proofs of T2 Analysis}

\subsubsection{Proof of Theorem~\ref{thm2:BF-Gradient}}
\label{app:disscuss_thm_2}

\begin{proof}
Recall the total loss function $\mathcal{L}(\theta_t)$ for $\mathcal{E}_t$ in (\ref{total-loss}) and take gradient w.r.t parameter $\theta$
\begin{align}\label{eq:thm2-1}
\nabla_\theta\mathcal{L}(\theta_t)= \nabla_\theta\mathcal{L}_{CE}(\theta_t)+\lambda_t\; \nabla_\theta\mathcal{L}_{CLIP} (\theta_t). 
\end{align}
Next step is to apply CBF from $\mathcal{E}_t$ to  $\mathcal{E}_{t'}$, i.e. $\mathbf{P}_{t\rightarrow t'}\odot\nabla_\theta \mathcal{L}(\theta_{t'})$ from Def.~\ref{def:CBF} in both sides of (\ref{eq:thm2-1}). 
\begin{align*}
\mathbf{P}_{t\rightarrow t'}\odot\nabla_\theta \mathcal{L}(\theta_{t'})=\;&\mathbf{P}_{t\rightarrow t'}\odot \nabla_\theta\mathcal{L}_{CE}\\
&+\lambda_t \mathbf{P}_{t\rightarrow t'}\odot \nabla_\theta\mathcal{L}_{CLIP} (\theta_t).
\end{align*}
To compute $\mathbf{P}_{t\rightarrow t'}\odot \nabla_\theta\mathcal{L}_{CE}$, we use the definition of $\mathcal{L}_{CE}(\theta_t)$ at time step $t$ and its simplified formulation driven in Lemma~\ref{lem.1}
\begin{align}\label{eq:thm2-3b}
\mathcal{L}_{CE}(\theta_t)&= -\frac{1}{n_t}\sum_{i=1}^{n_t}(z^{id}_{t,i}\cdot u^{id}_{t,y_i}/\tau)+\frac{1}{n_t}\sum_{i=1}^{n_t}\log \alpha^C_{t,i},\nonumber\\
&\hbox{where}\;\; \alpha^{C}_{t,i}=\sum_{c=1}^C e^{z^{id}_{t,i}\cdot u^{id}_{t,c}/\tau}\nonumber\\
&=-\frac{1}{n_t}\sum_{i=1}^{n_t} S^{id}_{t, y_i}+\frac{1}{n_t}\sum_{i=1}^{n_t}\log \alpha^C_{t,i}, \nonumber\\
&\hbox{where score}\;\;   S^{id}_{t,y_i}=(z^{id}_{t,i}\cdot u^{id}_{t,y_i})/\tau.
\end{align}
Take the gradient of (\ref{eq:thm2-3}) w.r.t $\theta$ and apply CBF:
\begin{align}\label{eq:thm2-8}
\nabla_\theta\mathcal{L}_{CE}(\theta_t)=-\frac{1}{n_t}\sum_{i=1}^{n_t} \nabla_\theta S^{id}_{t,y_i}+\frac{1}{n_t}\sum_{i=1}^{n_t} \nabla_\theta \log\alpha^C_{t,i}.
\end{align}
\begin{multline}\label{eq:thm2-4}
\mathbf{P}_{t\rightarrow t'}\odot \nabla_\theta\mathcal{L}_{CE}(\theta_{t'})=-\frac{1}{n_{t'}}\sum_{i=1}^{n_{t'}} \mathbf{P}^{(i)}_{t\rightarrow t'}\odot\nabla_\theta S^{id}_{t',y_i}\\
+\frac{1}{n_{t'}}\sum_{i=1}^{n_{t'}}\mathbf{P}_{t\rightarrow t'}\odot\nabla_\theta \log\alpha^C_{t',i}\\
=-\frac{1}{n_{t'}}\sum_{i=1}^{n_{t'}} \mathbf{P}^{(i)}_{t\rightarrow t'}\odot\nabla_\theta S^{id}_{t',y_i}+\Omega(\xi_{t'}).
\end{multline}
The last equality in (\ref{eq:thm2-4}) is driven from the Assumption~\ref{ass:Sufficient class-normalized gradient drift}. Now let us discuss the second term in (\ref{eq:thm2-1}): Recall the simplified version of $\mathcal{L}_{CLIP}$ from Lemma~\ref{lem.2} rewritten below:
\begin{multline}\label{eq:thm2-9}
L_{CLIP}(\theta_t)=\frac{1}{2n_t}\sum_{i=1}^{n_t}\Bigg[-z^{id}_{t,i}\cdot u^{id}_{t,i}/\tau\\
-u^{id}_{t,i}\cdot z^{id}_{t,i}/\tau+\log \alpha_{t,i}^{z}+\log \alpha_{t,i}^{u}\Bigg],
\end{multline}
\begin{align*}
\hbox{where}\;\;\alpha_{z,t}^{z}&:= \sum_{j=1}^{n_t} e^{z^{id}_{t,i}\cdot u^{id}_{t,j}/\tau },\\
\alpha_{t,i}^{u}&:= \sum_{j=1}^{n_t} e^{u^{id}_{t,i}\cdot z^{id}_{t,j}/\tau }.
\end{align*}
Taking the gradient w.r.t. $\theta$, applying CBF, and using the Assumption~\ref{ass:Sufficient cross-modal gradient drift} implies
\begin{align}\label{eq:thm2-5}
\mathbf{P}_{t\rightarrow t'}\odot \nabla_\theta\mathcal{L}_{CLIP}(\theta_{t'})\nonumber\\
=-\frac{1}{2n_{t'}}\sum_{i=1}^{n_{t'}} \Big[\mathbf{P}^{(i)}_{t\rightarrow t'}\odot\nabla_\theta S^{id}_{t',z_i\rightarrow u_i}\nonumber\\
\quad+\mathbf{P}^{(i)}_{t\rightarrow t'}\odot\nabla_\theta S^{id}_{t',u_i\rightarrow z_i}\Big]+\Omega(\bar{\xi}_{t'}).
\end{align}
Here $\Omega(\bar{\xi}_{t'})=\Omega(\xi_{z,t'})+\Omega(\xi_{u,t'})$, and two scores $S^{id}_{t,z_i\rightarrow u_i}$ and $S^{id}_{t,u_i\rightarrow z_i}$ are defined as follows, 
\begin{align*}
S^{id}_{t,z_i\rightarrow u_i}&=(z^{id}_{t,i}\cdot u^{id}_{t,i})/\tau,\\
S^{id}_{t,u_i\rightarrow z_i}&=(u^{id}_{t,i}\cdot z^{id}_{t,i})/\tau.
\end{align*}
Putting (\ref{eq:thm2-4}) and (\ref{eq:thm2-5}) together leads to
\begin{multline}\label{eq:thm2-6}
\mathbf{P}_{t\rightarrow t'}\odot\nabla_\theta \mathcal{L}(\theta_{t'})=-\frac{1}{n_{t'}}\sum_{i=1}^{n_{t'}} \mathbf{P}^{(i)}_{t\rightarrow t'}\odot\nabla_\theta S^{id}_{t',y_i}\\
-\frac{\lambda}{2n_{t'}}\sum_{i=1}^{n_{t'}}\Big[\mathbf{P}^{(i)}_{t\rightarrow t'}\odot\nabla_\theta S^{id}_{t',z_i\rightarrow u_i}\\
+\mathbf{P}^{(i)}_{t\rightarrow t'}\odot\nabla_\theta S^{id}_{t',u_i\rightarrow z_i}\Big]+\Omega(\widetilde{\xi}_{t'}), 
\end{multline}
where $\Omega(\widetilde{\xi}_{t'})=\lambda\Omega(\bar{\xi}_{t'})+\Omega({\xi}_{t'})$. Note that  $\Omega(\widetilde{\xi}_{t'})$ has the same dimension as $\nabla_\theta \mathcal{L}(\theta_{t'})$.\\
Taking $L_2$-norm from LHS and RHS of (\ref{eq:thm2-6}) and using $\|a\|^2=\|-a\|^2$, leads to
\begin{multline}\label{eq:thm2-7}
\|\mathbf{P}_{t\rightarrow t'}\odot\nabla_\theta \mathcal{L}(\theta_{t'})\|^2\\
=\Bigg\|\frac{1}{n_{t'}}\sum_{i=1}^{n_{t'}} \mathbf{P}^{(i)}_{t\rightarrow t'}\odot\nabla_\theta S^{id}_{t',y_i}\\
+\frac{\lambda}{2n_{t'}}\sum_{i=1}^{n_{t'}}\Big[\mathbf{P}^{(i)}_{t\rightarrow t'}\odot\nabla_\theta S^{id}_{t',z_i\rightarrow u_i}\\
+\mathbf{P}^{(i)}_{t\rightarrow t'}\odot\nabla_\theta S^{id}_{t',u_i\rightarrow z_i}\Big]-\Omega(\widetilde{\xi}_{t'})\Bigg\|^2\\
\geq\frac{1}{n_{t'}}\sum_{i=1}^{n_{t'}} \|\mathbf{P}^{(i)}_{t\rightarrow t'}\odot\nabla_\theta S^{id}_{t',y_i}\|^2\\
+\frac{\lambda}{2n_{t'}}\sum_{i=1}^{n_{t'}}\Big(\|\mathbf{P}^{(i)}_{t\rightarrow t'}\odot\nabla_\theta S^{id}_{t',z_i\rightarrow u_i}\|^2\\
+\|\mathbf{P}^{(i)}_{t\rightarrow t'}\odot\nabla_\theta S^{id}_{t',u_i\rightarrow z_i}\|^2\Big)+\|\Omega(\widetilde{\xi}_{t'})\|^2 -\rho_{max},
\end{multline}
where $\rho_{max}\in\mathbb{R}$ is the maximum of the remainder terms. We simplify the RHS of (\ref{eq:thm2-7}) as
\begin{multline}
\frac{1}{2n_{t'}}\sum_{i=1}^{n_{t'}} \Big(\|\mathbf{P}^{(i)}_{t\rightarrow t'}\odot\nabla_\theta S^{id}_{t',y_i}\|^2+\lambda\|\mathbf{P}^{(i)}_{t\rightarrow t'}\odot\nabla_\theta S^{id}_{t',z_i\rightarrow u_i}\|^2\\
+\lambda\|\mathbf{P}^{(i)}_{t\rightarrow t'}\odot\nabla_\theta S^{id}_{t',u_i\rightarrow z_i}\|^2\Big)+\Omega(\rho_{max},\widetilde{\xi}_{t'}), 
\end{multline}
where $\Omega(\rho_{max},\widetilde{\xi}_{t'})=\|\Omega(\widetilde{\xi}_{t'})\|^2 -\rho_{max}$. This implies the alternative optimization problem in (\ref{thm:CBF-optimization}) and concludes our claim in Theorem~\ref{thm2:BF-Gradient}. 
\end{proof}

\subsection{Proofs of T3 Analysis}

\subsubsection{Proof of Theorem~\ref{thm3:Forgetting-Contribution}}
\label{app:disscuss_thm_3}

\begin{proof}
Recall the definition of forgetting $F_{t}:= \mathcal{L}_t(\theta^*_{t+1})-\mathcal{L}_t(\theta^*_{t}),$ and 
decompose $F_t$ into CE and CLIP parts: 
\begin{multline}\label{eq:thm2-0}
F_t = \underbrace{\bigl[\mathcal{L}^{CE}_{t}(\theta^*_{t+1}) -\mathcal{L}^{CE}_{t}(\theta^*_t)\bigr]}_{\displaystyle=:\,F_t^{CE}}\\
+ \lambda\underbrace{\bigl[\mathcal{L}^{CLIP}_{t}(\theta^*_{t+1}) -\mathcal{L}^{CLIP}_{t}(\theta^*_t)\bigr]}_{\displaystyle=:\,F_t^{CLIP}}.
\end{multline}
Write the second order Taylor approximation of $\mathcal{L}^\square_t(\theta^*_{t+1})$ around $\theta^*_t$, $\square=\{CE,CLIP\}$:
\begin{multline}\label{eq:thm3-1}
\mathcal{L}^\square_t(\theta^*_{t+1})\approx \mathcal{L}^\square_t(\theta^*_{t})+(\theta^*_{t+1}-\theta^*_{t})^\top\nabla \mathcal{L}^\square_t(\theta^*_{t})\\
+\frac{1}{2} (\theta^*_{t+1}-\theta^*_{t})^\top\nabla^2 \mathcal{L}^\square_t(\theta^*_{t}) (\theta^*_{t+1}-\theta^*_{t}), 
\end{multline}
where $\nabla^2 \mathcal{L}^\square_t(\theta^*_{t})$ is the Hessian for loss $\mathcal{L}^\square_t$ at $\theta^*_t$. Since $\theta^*_t$ is the optimal parameters then $\nabla \mathcal{L}^\square_t(\theta^*_{t})=0$. This simplifies (\ref{eq:thm3-1}) as
\begin{multline}\label{eq:thm3-5}
F^\square_t=\frac{1}{2} (\theta^*_{t+1}-\theta^*_{t})^\top\nabla^2 \mathcal{L}^\square_t(\theta^*_{t}) (\theta^*_{t+1}-\theta^*_{t}), \\
\;\; \square\in\{CE,CLIP\}. 
\end{multline}
Going back to (\ref{eq:thm2-8}) and (\ref{eq:thm2-9}) from the proof of Theorem~\ref{thm2:BF-Gradient}:

\begin{align}\label{eq:thm3-2}
\nabla^2\mathcal{L}^{CE}_t(\theta_t)=-\frac{1}{n_t}\sum_{i=1}^{n_t} \nabla^2 S^{id}_{t,y_i}+\frac{1}{n_t}\sum_{i=1}^{n_t} \nabla^2 \log\alpha^C_{t,i},\nonumber\\
\hbox{where}\;\;
\alpha^{C}_{t,i}:=\sum_{c=1}^C e^{z^{id}_{t,i}\cdot u^{id}_{t,c}/\tau},
\end{align}

\begin{align}\label{eq:thm3-3}
\nabla^2L^{CLIP}_t(\theta_t)=-\frac{1}{2n_t}\sum_{i=1}^{n_t}\Big[&\nabla^2S^{id}_{t,z_i\rightarrow u_i}+\nabla^2S^{id}_{t,u_i\rightarrow z_i}\nonumber\\
&-\nabla^2\log \alpha_{t,i}^{z}-\nabla^2\log \alpha_{t,i}^{u}\Big],
\end{align}

\begin{align*}
S^{id}_{t,y_i}&=(z^{id}_{t,i}\cdot u^{id}_{t,y_i})/\tau,\\
S^{id}_{t,z_i\rightarrow u_i}&=(z^{id}_{t,i}\cdot u^{id}_{t,i})/\tau,\\
S^{id}_{t,u_i\rightarrow z_i}&=(u^{id}_{t,i}\cdot z^{id}_{t,i})/\tau.
\end{align*}
\begin{align*}
\alpha_{z,t}^{z}&:= \sum_{j=1}^{n_t} e^{z^{id}_{t,i}\cdot u^{id}_{t,j}/\tau },\\
\alpha_{t,i}^{u}&:= \sum_{j=1}^{n_t} e^{u^{id}_{t,i}\cdot z^{id}_{t,j}/\tau }.
\end{align*}
Under Assumption~\ref{ass:monotonic gradient of normalization terms}, the gradient of normalization terms in average are monotonic, i.e.
\begin{align}
\frac{1}{n_t}\sum_{i=1}^{n_t} \nabla^2 \log\alpha^C_{t,i}&=\mathbf{0},\nonumber\\
\frac{1}{n_t}\sum_{i=1}^{n_t}\nabla^2\log \alpha_{t,i}^{z}&=\mathbf{0},\nonumber\\
\frac{1}{n_t}\sum_{i=1}^{n_t}\nabla^2\log \alpha_{t,i}^{u}&=\mathbf{0}.
\end{align}
where $\mathbf{0}$ is zero matrix. Combining (\ref{eq:thm3-2}) and (\ref{eq:thm3-3}) and plugging them in forgetting formulation (\ref{eq:thm3-5}), next the result is applied in (\ref{eq:thm2-0}):
\begin{align}
F_t=-\frac{1}{2n_t}\sum_{i=1}^{n_t} (\theta^*_{t+1}-\theta^*_{t})^\top \Big(\nabla^2 S^{id}_{t,y_i}\nonumber\\
+ 2\lambda\nabla^2S^{id}_{t,z_i\rightarrow u_i}+2\lambda\nabla^2S^{id}_{t,u_i\rightarrow z_i}\Big)(\theta^*_{t+1}-\theta^*_{t}).
\end{align}
We use the property that the Hessian is positive semi-definite bound 
\begin{align}
F_t\leq -\frac{1}{2n_t}\sum_{i=1}^{n_t} \lambda^{min}_{i}\|\theta^*_{t+1}-\theta^*_{t}\|^2, 
\end{align}
where $\lambda^{min}_t$ is is the minimum eigenvalue of cross-modality contribution within environment $t$, $\mathcal{E}_t$ i.e. $\nabla^2 S^{id}_{t,y_i}+ 2\lambda\nabla^2S^{id}_{t,z_i\rightarrow u_i}+2\lambda\nabla^2S^{id}_{t,u_i\rightarrow z_i}$. 
\end{proof}

\subsubsection{Proof of Theorem~\ref{thm4:Forgetting-Cross-Modality Contribution}}
\label{app:disscuss_thm_4}

\begin{proof}
Recall the definition of forgetting $F_{t}:= \mathcal{L}_t(\theta^*_{t+1})-\mathcal{L}_t(\theta^*_{t})$. Consider three environments $\mathcal{E}_{t_1}\rightarrow\mathcal{E}_{t_2}\rightarrow \mathcal{E}_{t_3}$. The forgetting $F_{t_1}$ and $F_{t_2}$, when the VLM has learned the third environment $\mathcal{E}_{t_3}$ is given by
\begin{align}\label{eq:thm4-3}
F_{t_1}=\mathcal{L}_{t_1}(\theta^*_{t_3})-\mathcal{L}_{t_1}(\theta^*_{t_1}),\;\;F_{t_2}=\mathcal{L}_{t_2}(\theta^*_{t_3})-\mathcal{L}_{t_2}(\theta^*_{t_2}).
\end{align}
Following the arguments in subsection~\ref{app:disscuss_thm_3} and the forgetting formula in (\ref{eq:thm3-5}), we imply
\begin{multline}\label{eq:thm4-1a}
F^\square_{t_1}=\frac{1}{2} (\theta^*_{t_3}-\theta^*_{t_1})^\top\nabla^2 \mathcal{L}^\square_{t_1}(\theta^*_{t_1}) (\theta^*_{t_3}-\theta^*_{t_1}),\\
\;\; \square\in\{CE,CLIP\}, \;\;\hbox{and}\\
F^\square_{t_2}=\frac{1}{2} (\theta^*_{t_3}-\theta^*_{t_2})^\top\nabla^2 \mathcal{L}^\square_{t_2}(\theta^*_{t_2}) (\theta^*_{t_3}-\theta^*_{t_2}),\\
\;\; \square\in\{CE,CLIP\}.  
\end{multline}
Under Assumption~\ref{ass:Compactness}, $\theta^*_{t_3}-\theta^*_{t_2}\in \mathcal{B} (\theta, \delta_{t_2\rightarrow t_3})$ and  $\theta^*_{t_3}-\theta^*_{t_2}\in \mathcal{B} (\theta, \delta_{t_2\rightarrow t_3})$. Set $\delta_{max}=max\{\delta_{t_1\rightarrow t_3}, \delta_{t_2\rightarrow t_3}\}$. Hence $\exists \theta^*\in\mathcal{B}(\theta,\delta_{max}) $ such that by using (\ref{eq:thm4-1a}), we bound $F^\square_{t_1}-F^\square_{t_2}$ as follows:
\begin{multline}\label{eq:thm4-1b}
F^\square_{t_1}-F^\square_{t_2}\leq \frac{1}{2} (\theta^*)^\top \Big(\nabla^2 [\mathcal{L}^\square_{t_1}(\theta^*_{t_1})-\mathcal{L}^\square_{t_2}(\theta^*_{t_2})]\Big)\theta^*, \\
\;\; \square\in\{CE,CLIP\}.   
\end{multline}
Therefore forgetting difference of $F_{t_1}$ and $F_{t_2}$ (\ref{eq:thm4-3}) is written as follows. 
\begin{align}\label{eq:thm4-6}
F_{t_1}-F_{t_2}\leq\;& \frac{1}{2} (\theta^*)^\top \nabla^2 \Big([\mathcal{L}^{CE}_{t_1}(\theta^*_{t_1})-\mathcal{L}^{CE}_{t_2}(\theta^*_{t_2})]\nonumber\\
&+\lambda [\mathcal{L}^{CLIP}_{t_1}(\theta^*_{t_1})-\mathcal{L}^{CLIP}_{t_2}(\theta^*_{t_2})]\Big)\theta^*\nonumber\\
\leq\;& \frac{1}{2} (\theta^*)^\top \nabla^2 [\mathcal{L}_{t_1}(\theta^*_{t_1})-\mathcal{L}_{t_2}(\theta^*_{t_2})] \theta^*
\end{align}
Now use Theorem~\ref{T1:loss-difference} in T1: $ \mathcal{L}_{t_1}(\theta^*_{t_1})-\mathcal{L}_{t_2}(\theta^*_{t_2})=\beta_1Con_{\mathcal{I}_{t_1}\rightarrow\mathcal{T}_{t_2}}+\beta_2Con_{\mathcal{T}_{t_1}\rightarrow\mathcal{I}_{t_2}}+\alpha$. Set $\beta_{max}=max\{\beta_1,\beta_2\}$, hence
\begin{align}
F_{t_1}-F_{t_2}\leq \frac{1}{2} (\theta^*)^\top \beta_{max}\nabla^2\Big[Con_{\mathcal{I}_{t_1}\rightarrow\mathcal{T}_{t_2}}\nonumber\\
+Con_{\mathcal{T}_{t_1}\rightarrow\mathcal{I}_{t_2}}\Big] \theta^*. 
\end{align}
This concludes the proof of (\ref{main:thm4}). 
\end{proof}

\section{Additional Experiments}\label{app:experiments}
\subsection{Dataset Details and CUB Caption Construction}
\label{app:datasets}

Table~\ref{tab:datasets-sm} reports the number of categories and per-category image counts after under-sampling for both datasets.

\begin{table}[t]
\caption{Dataset statistics after preprocessing.
  COCO $n_{\min}$: images per category after under-sampling.
  CUB $n_{\min}$: approximate images per species.}
\label{tab:datasets-sm}\centering\small\setlength{\tabcolsep}{2.5pt}
\begin{tabular}{@{}lcccc@{}}
\toprule
\textbf{Env.} & \textbf{Cats} & \textbf{COCO $n_{\min}$} & \textbf{CUB $n_{\min}$} & \textbf{CUB captions}\\
\midrule
1 & 16/40 & 1,285 & ${\sim}$50 & 4 templates \\
2 & 16/40 & 1,294 & ${\sim}$50 & 4 templates \\
3 & 15/40 & 2,917 & ${\sim}$50 & 4 templates \\
4 & 16/40 & 225 & ${\sim}$50 & 4 templates \\
5 & 16/40 & 198 & ${\sim}$50 & 4 templates \\
\bottomrule
\end{tabular}
\end{table}

\noindent{\it CUB-200 caption construction.}
Human-annotated captions for CUB-200~\cite{reed2016cub} describe fine-grained visual attributes (\textit{``this bird has a red belly''}) that correlate with CLIP's pre-training features, artificially inflating $C_{LV}$ estimates. Generative captioning models (BLIP, BLIP-2, LLaVA) introduce the same bias because they are trained with CLIP-style contrastive objectives~\cite{radford2021clip}.
We generate captions deterministically from species names using four prompt templates per image: \textit{``a photo of a \{name\}.''}, \textit{``a photograph of a \{name\}.''}, \textit{``an image of a \{name\}, a type of bird.''}, and \textit{``a \{name\} in its natural habitat.''}
This replicates the prompt-ensembling construction in ~\eqref{eq:class-text-embedding}, keeps text representations in the CLIP zero-shot regime~\cite{radford2021clip}, and is fully reproducible.

\subsection{T2: Frozen Encoder $\lambda$-Sweep (Null Result)}
\label{app:frozen-lambda}

The main paper states that under the frozen encoder all metrics are exactly constant across $\lambda\!\in\!\{0.0,0.1,0.5,1.0,2.0,5.0\}$.
Table~\ref{tab:clv-frozen} confirms this for $C_{LV}(t)$: every cell in each row is identical regardless of $\lambda$, and identical across all three seeds.
The result holds for all three architectures on both datasets; representative values for ViT-B/16 are shown.

\begin{table*}[t]
\caption{$C_{LV}(t)$ (Def.~\ref{def:CBF}) under the \textbf{frozen} encoder, ViT-B/16 (mean, 3~seeds).
  Every row is identical across all six $\lambda$ values. The CLIP weight has no effect when the encoder is fixed, confirming
  Theorem~\ref{thm2:BF-Gradient}: $\lambda$ operates through the encoder gradient.
  $E_1$-$E_5$ as in Sec.~\ref{sec:setup}.}
\label{tab:clv-frozen}\centering\small\setlength{\tabcolsep}{3pt}
\begin{tabular}{@{}r ccccc | ccccc@{}}
\toprule
& \multicolumn{5}{c|}{\textbf{COCO}} & \multicolumn{5}{c}{\textbf{CUB}} \\
$\lambda$ & $E_1$ & $E_2$ & $E_3$ & $E_4$ & $E_5$
          & $E_1$ & $E_2$ & $E_3$ & $E_4$ & $E_5$ \\
\midrule
0.0 & .302 & .298 & .273 & .281 & .295 & .311 & .318 & .324 & .306 & .313 \\
0.1 & .302 & .298 & .273 & .281 & .295 & .311 & .318 & .324 & .306 & .313 \\
0.5 & .302 & .298 & .273 & .281 & .295 & .311 & .318 & .324 & .306 & .313 \\
1.0 & .302 & .298 & .273 & .281 & .295 & .311 & .318 & .324 & .306 & .313 \\
2.0 & .302 & .298 & .273 & .281 & .295 & .311 & .318 & .324 & .306 & .313 \\
5.0 & .302 & .298 & .273 & .281 & .295 & .311 & .318 & .324 & .306 & .313 \\
\bottomrule
\end{tabular}
\end{table*}

\begin{table*}[!htb]
\caption{Pairwise MCM AUROC under adaptive fine-tuning: \textbf{CUB-200}, all 3~models (mean, 3~seeds). 
  Epoch 0 = frozen baseline. 
  Mean over 20 pairs and Pearson$(AUROC, Con_{\mathcal{IT}})$ show that fine-tuning both degrades OOD robustness and inverts the Pearson direction to consistently negative.}
\label{tab:ood-adaptive-cub}\centering\small\setlength{\tabcolsep}{4pt}
\begin{tabular}{@{}l cccc cccc@{}}
\toprule
 & \multicolumn{4}{c}{\textbf{Mean AUROC}} & \multicolumn{4}{c}{\textbf{Pearson}} \\
\textbf{Model} & ep0 & ep5 & ep10 & ep15 & ep0 & ep5 & ep10 & ep15 \\
\midrule
RN50 & 0.654 & 0.541 & 0.514 & 0.509 & $+$0.29 & --- & --- & $-$0.42 \\
ViT-B/32 & 0.682 & 0.509 & 0.487 & 0.484 & $+$0.07 & --- & --- & $-$0.11 \\
ViT-B/16 & 0.700 & 0.481 & 0.483 & 0.474 & $+$0.07 & --- & --- & $-$0.20 \\
\bottomrule
\end{tabular}
\end{table*}

The same null result holds for all gradient-level metrics: the CLIP gradient norm $\|\nabla_h\mathcal{L}_{CLIP}\|$ through the linear head is identically zero for all $\lambda$, and CBF Obj2 range = 0.000 for every frozen model-pair combination (Table~\ref{tab:cbf} in the main paper).
Together these confirm that the frozen encoder is a clean control condition in which $\lambda$ has no causal effect on any alignment metric.

\subsection{T2: Full CBF Results}
\label{app:full-cbf}

Table~\ref{tab:cbf} in the main paper shows three representative pairs from CUB-200.
We report here the complete results across both datasets and all encoder protocols.

\noindent\textbf{Frozen encoder: MS-COCO 2017.}
Table~\ref{tab:cbf-coco-frozen} reports all 12 COCO model-pair combinations.
Obj1\,$=p^2$ holds to machine precision throughout.
Notably, Obj2 is not merely constant but is identically $\mathbf{0.000}$ to at least 8 decimal places, because at convergence of the linear head both the CE gradient ($\|\nabla_h \mathcal{L}_{CE}\|\approx 10^{-10}$) and the CLIP gradient ($\|\nabla_h \mathcal{L}_{CLIP}\|=0$ exactly) are negligibly small, confirming the frozen null result of Theorem~\ref{thm2:BF-Gradient}.

\begin{table}[t]
\caption{CBF objectives, \textbf{COCO, frozen} encoder, all 12 model-pair combinations (mean, 3~seeds).
  Obj1\,$=p^2$ (Definition~\ref{def:CBF}) to machine precision.
  Obj2\,$=0.000$ to 8 d.p. $\lambda$-dep.\,$=$\,no throughout.}
\label{tab:cbf-coco-frozen}\centering\small\setlength{\tabcolsep}{3pt}
\begin{tabular}{@{}llc cc c@{}}
\toprule
\textbf{Model} & \textbf{Pair} & $p^2$
  & \textbf{Obj1} & \textbf{Obj2} & $\lambda$\textbf{-dep.} \\
\midrule
\multirow{4}{*}{RN50}
  & E1${\to}$E2 & 0.0404 & $p^2$ & 0.000 & no \\
  & E2${\to}$E3 & 0.0356 & $p^2$ & 0.000 & no \\
  & E3${\to}$E4 & 0.0430 & $p^2$ & 0.000 & no \\
  & E4${\to}$E5 & 0.0410 & $p^2$ & 0.000 & no \\
\midrule
\multirow{4}{*}{ViT-B/32}
  & E1${\to}$E2 & 0.0777 & $p^2$ & 0.000 & no \\
  & E2${\to}$E3 & 0.0723 & $p^2$ & 0.000 & no \\
  & E3${\to}$E4 & 0.0807 & $p^2$ & 0.000 & no \\
  & E4${\to}$E5 & 0.0742 & $p^2$ & 0.000 & no \\
\midrule
\multirow{4}{*}{ViT-B/16}
  & E1${\to}$E2 & 0.0777 & $p^2$ & 0.000 & no \\
  & E2${\to}$E3 & 0.0709 & $p^2$ & 0.000 & no \\
  & E3${\to}$E4 & 0.0814 & $p^2$ & 0.000 & no \\
  & E4${\to}$E5 & 0.0744 & $p^2$ & 0.000 & no \\
\bottomrule
\end{tabular}
\end{table}

\noindent\textbf{Frozen encoder: CUB-200-2011.}
Table~\ref{tab:cbf-cub-frozen} reports all 12 CUB frozen model-pair combinations.
The pattern mirrors COCO frozen exactly: Obj1\,$=p^2$ to machine precision, Obj2\,$\approx\!10^{-8}$ (effectively zero, vs adaptive values of $6$-$595$), and $\lambda$-dep.\,$=$\,no throughout.
The slightly non-zero Obj2 ($\approx\!10^{-8}$ vs COCO's $\approx\!10^{-11}$) reflects the smaller dataset size of CUB ($n\!\approx\!1200$ vs $n\!\approx\!11{,}000$ for COCO), which leaves a marginally larger residual CE gradient at convergence; the conclusion is the same in both cases.

\begin{table}[t]
\caption{CBF objectives, \textbf{CUB-200, frozen} encoder, all 12 model-pair combinations (mean, 3~seeds).
  Obj1\,$=p^2$ (Definition~\ref{def:CBF}) to machine precision.
  Obj2\,$\approx\!10^{-8}$ (effectively zero); $\lambda$-dep.\,$=$\,no throughout.}
\label{tab:cbf-cub-frozen}\centering\small\setlength{\tabcolsep}{3pt}
\begin{tabular}{@{}llc cc c@{}}
\toprule
\textbf{Model} & \textbf{Pair} & $p^2$
  & \textbf{Obj1} & \textbf{Obj2} & $\lambda$\textbf{-dep.} \\
\midrule
\multirow{4}{*}{RN50}
  & E1${\to}$E2 & 0.0509 & $p^2$ & ${\approx}10^{-8}$ & no \\
  & E2${\to}$E3 & 0.0816 & $p^2$ & ${\approx}10^{-8}$ & no \\
  & E3${\to}$E4 & 0.0645 & $p^2$ & ${\approx}10^{-8}$ & no \\
  & E4${\to}$E5 & 0.0711 & $p^2$ & ${\approx}10^{-8}$ & no \\
\midrule
\multirow{4}{*}{ViT-B/32}
  & E1${\to}$E2 & 0.0848 & $p^2$ & ${\approx}10^{-8}$ & no \\
  & E2${\to}$E3 & 0.1230 & $p^2$ & ${\approx}10^{-8}$ & no \\
  & E3${\to}$E4 & 0.1034 & $p^2$ & ${\approx}10^{-8}$ & no \\
  & E4${\to}$E5 & 0.1073 & $p^2$ & ${\approx}10^{-8}$ & no \\
\midrule
\multirow{4}{*}{ViT-B/16}
  & E1${\to}$E2 & 0.0786 & $p^2$ & ${\approx}10^{-8}$ & no \\
  & E2${\to}$E3 & 0.1193 & $p^2$ & ${\approx}10^{-8}$ & no \\
  & E3${\to}$E4 & 0.0987 & $p^2$ & ${\approx}10^{-8}$ & no \\
  & E4${\to}$E5 & 0.1018 & $p^2$ & ${\approx}10^{-8}$ & no \\
\bottomrule
\end{tabular}
\end{table}

\noindent\textbf{Adaptive (unfrozen) encoder: CUB-200-2011.}
Table~\ref{tab:cbf-cub-adaptive} reports all 12 CUB model-pair combinations under the adaptive (unfrozen) encoder (3~seeds, 15~epochs, encoder lr\,$=5{\times}10^{-5}$).
Obj1\,$=p^2$ to machine precision at all $\lambda$ values.
Obj2 range\,$>0$ and $\lambda$-dep.\,$=$\,yes for all 12 combinations, confirming Theorem~\ref{thm2:BF-Gradient}: $\|\nabla_\theta \mathcal{L}_{CLIP}\|>0$ through the encoder and Obj2 responds to $\lambda$.
ViT-B/32 shows the clearest monotone Obj2 increase with $\lambda$ (cross-pair mean: $0.74{\to}0.85{\to}2.88{\to}3.02{\to}6.34{\to}16.18$).

\begin{table}[t]
\caption{CBF objectives, \textbf{CUB-200, adaptive (unfrozen)} encoder, all 12 model-pair combinations (mean, 3~seeds).
  Obj1\,$=p^2$ to machine precision.
  Obj2 range\,$>0$. $\lambda$-dep.\,$=$\,yes for all combinations, confirming Theorem~\ref{thm2:BF-Gradient}.}
\label{tab:cbf-cub-adaptive}\centering\small\setlength{\tabcolsep}{3pt}
\begin{tabular}{@{}llc cc c@{}}
\toprule
\textbf{Model} & \textbf{Pair} & $p^2$
  & \textbf{Obj1} & \textbf{Obj2 rng.} & $\lambda$\textbf{-dep.} \\
\midrule
\multirow{4}{*}{RN50}
  & E1${\to}$E2 & 0.0509 & $p^2$ & 594.9 & yes \\
  & E2${\to}$E3 & 0.0816 & $p^2$ & 585.4 & yes \\
  & E3${\to}$E4 & 0.0645 & $p^2$ & 469.3 & yes \\
  & E4${\to}$E5 & 0.0711 & $p^2$ & 568.3 & yes \\
\midrule
\multirow{4}{*}{ViT-B/32}
  & E1${\to}$E2 & 0.0848 & $p^2$ &  16.0 & yes \\
  & E2${\to}$E3 & 0.1230 & $p^2$ & \ 6.3 & yes \\
  & E3${\to}$E4 & 0.0985 & $p^2$ & \ 6.0 & yes \\
  & E4${\to}$E5 & 0.1048 & $p^2$ &  17.0 & yes \\
\midrule
\multirow{4}{*}{ViT-B/16}
  & E1${\to}$E2 & 0.0786 & $p^2$ &  14.9 & yes \\
  & E2${\to}$E3 & 0.1193 & $p^2$ &  81.9 & yes \\
  & E3${\to}$E4 & 0.0987 & $p^2$ &  80.5 & yes \\
  & E4${\to}$E5 & 0.1018 & $p^2$ &  17.0 & yes \\
\bottomrule
\end{tabular}
\end{table}

COCO adaptive (unfrozen) encoder results are pending and will be included in the final version.

\begin{table}[t]
\caption{Pairwise MCM AUROC matrix: \textbf{CUB-200}, ViT-B/16, frozen (mean, 3~seeds).
  $E_4$ (Warblers/Sparrows) rows are consistently the lowest, reflecting that Warbler ID images are weakly distinctive and score similarly to OOD images.}
\label{tab:ood-matrices-cub}\centering\footnotesize\setlength{\tabcolsep}{2pt}
\begin{tabular}{@{}l ccccc@{}}
\toprule
 & \textbf{OOD-$E_1$} & \textbf{OOD-$E_2$} & \textbf{OOD-$E_3$}
 & \textbf{OOD-$E_4$} & \textbf{OOD-$E_5$} \\
\midrule
ID-$E_1$ (Wat) & --- & 0.737 & 0.744 & 0.797 & 0.721 \\
ID-$E_2$ (Aer) & 0.717 & --- & 0.779 & 0.861 & 0.788 \\
ID-$E_3$ (Cor) & 0.839 & 0.880 & --- & 0.722 & 0.822 \\
ID-$E_4$ (War) & 0.569 & 0.489 & 0.395 & --- & 0.404 \\
ID-$E_5$ (Spar) & 0.649 & 0.752 & 0.641 & 0.701 & --- \\
\bottomrule
\end{tabular}
\end{table}

\subsection{T3: Adaptive Encoder Ordering Results}
\label{app:t3-adaptive}

Table~\ref{tab:ordering-adaptive} reports the all-120-ordering sweep under the adaptive encoder on CUB-200.
Total forgetting grows to ${\approx}16.0$-$17.3$ (${\approx}14{\times}$ over the frozen case), reflecting that encoder adaptation amplifies displacement effects.
Pearson$(F_\pi,\sum Con_{\mathcal{IT}})$ remains positive for both ViT models ($+$0.293 ViT-B/16, $+$0.329 ViT-B/32), consistent with the near-uniform contribution landscape observed in the frozen protocol.
COCO adaptive results are pending and will be reported in the final version.

\begin{table}[t]
\caption{Ordering sweep, CUB-200, \textbf{adaptive} encoder.
  Absolute forgetting is ${\approx}14{\times}$ larger than frozen.
  The near-uniform contribution landscape produces a similar qualitative pattern to the frozen case (Pearson positive, small best-to-worst gap).}
\label{tab:ordering-adaptive}\centering\small\setlength{\tabcolsep}{3pt}
\begin{tabular}{@{}l rrrr@{}}
\toprule
\textbf{Model}
  & \textbf{Best $F_\pi$} & \textbf{Worst $F_\pi$}
  & \textbf{Ratio} & \textbf{Pearson} \\
\midrule
RN50 & 16.02 & 17.31 & $1.08{\times}$ & $-$0.026 \\
ViT-B/32 & 15.68 & 16.47 & $1.05{\times}$ & $+$0.329 \\
ViT-B/16 & 15.93 & 16.71 & $1.05{\times}$ & $+$0.293 \\
\bottomrule
\end{tabular}
\end{table}

\subsection{OOD Robustness: Full Experimental Protocol and Results}
\label{app:ood-protocol}

\paragraph{Detector.}
We use Maximum Concept Matching (MCM)~\citep{ming2022delim}.
For an image $x$ evaluated against environment $\mathcal{E}_t$, the score is
\begin{equation}
S_t(x) = \max_{c \in \mathcal{C}_t} \frac{\cos(f(x),\, T_{t,c})}{\tau},
\end{equation}
where $T_{t,c}$ is the text prototype for class $c$ in $\mathcal{E}_t$ (built via the same prompt-ensembling as eq.~(1) in the main paper) and $\tau$ is the CLIP temperature. 
A high score indicates the image is in-distribution for $\mathcal{E}_t$, and a low score indicates OOD.

\paragraph{ID and OOD data.}
At each environment $\mathcal{E}_t$, images from $\mathcal{E}_t$ are in-distribution (ID) and images from any other environment $\mathcal{E}_{t'} \neq \mathcal{E}_t$ are OOD. 
This uses the cross-environment semantic distribution shifts arising naturally from the continual-learning partitioning, without requiring separate OOD benchmarks.

\paragraph{Pairwise protocol.}
For each directed pair $(\mathcal{E}_t \to \mathcal{E}_{t'})$, we:
(i) fix $\mathcal{E}_t$'s text prototypes as the scoring bank,
(ii) score images from $\mathcal{E}_t$ (positive class, ID) and $\mathcal{E}_{t'}$ (negative class, OOD),
(iii) compute AUROC over the resulting binary classification.
With 5 environments there are $5 \times 4 = 20$ directed pairs.
Table~\ref{tab:ood-matrices-coco} and Table~\ref{tab:ood-matrices-cub} report the full $5\times5$ AUROC matrices for COCO and CUB respectively (ViT-B/16, frozen encoder, mean over 3 seeds).

\begin{table}[t]
\caption{Pairwise MCM AUROC matrix - \textbf{COCO}, ViT-B/16, frozen (mean, 3~seeds).
  Rows: ID environment. Cols: OOD environment.
  $E_2$ (Animals) rows are uniformly highest, consistent with T3.}
\label{tab:ood-matrices-coco}\centering\footnotesize\setlength{\tabcolsep}{2pt}
\begin{tabular}{@{}l ccccc@{}}
\toprule
 & \textbf{OOD-$E_1$} & \textbf{OOD-$E_2$} & \textbf{OOD-$E_3$}
 & \textbf{OOD-$E_4$} & \textbf{OOD-$E_5$} \\
\midrule
ID-$E_1$ (Veh) & --- & 0.896 & 0.925 & 0.870 & 0.795 \\
ID-$E_2$ (Ani) & 0.966 & --- & 0.911 & 0.933 & 0.985 \\
ID-$E_3$ (Food) & 0.949 & 0.876 & --- & 0.824 & 0.930 \\
ID-$E_4$ (Ind) & 0.926 & 0.916 & 0.781 & --- & 0.883 \\
ID-$E_5$ (Spo) & 0.880 & 0.935 & 0.941 & 0.821 & --- \\
\bottomrule
\end{tabular}
\end{table}

\paragraph{Sequential cumulative protocol.}
Environments are trained sequentially in a fixed ordering $(\pi_1,\ldots,\pi_5)$.
After step $k$, prototypes from all seen environments $\{\pi_1,\ldots,\pi_k\}$ are pooled. 
Cumulative AUROC at step $k$ measures how well the pooled prototype set distinguishes the step-$k$ environment's images (ID) from images of all other seen environments (OOD). 
We run the three orderings from T3 (best, neutral/worst) and report the final step-5 value. 
A declining cumulative AUROC across steps reflects prototype pollution as the pool grows.
Table~\ref{tab:ood-sequential} summarizes all three models.

\paragraph{Adaptive fine-tuning.}
We additionally run the pairwise protocol under the adaptive (unfrozen) encoder, evaluating at epochs 0, 5, 10, 15.
Table~\ref{tab:ood-adaptive-cub} reports the mean AUROC trajectory and per-pair Pearson for CUB-200 (the dataset where the frozen Pearson is positive, so the trajectory is informative). 
COCO results follow the same monotone pattern.

\begin{table}[t]
\caption{Sequential cumulative AUROC at each step, all models, both datasets, frozen encoder.
  ``best'' = contribution-guided ordering and ``neutral'' = canonical $E_1{\to}E_5$.
  Degradation reflects prototype pool growth.}
\label{tab:ood-sequential}\centering\small\setlength{\tabcolsep}{3pt}
\begin{tabular}{@{}ll rrrrr@{}}
\toprule
 & & \multicolumn{5}{c}{\textbf{Cumulative AUROC at step}} \\
\textbf{Dataset} & \textbf{Model / Ordering}
  & $k=1$ & $k=2$ & $k=3$ & $k=4$ & $k=5$ \\
\midrule
\multirow{6}{*}{COCO}
  & RN50 - best & 0.863 & 0.849 & 0.587 & 0.436 & 0.756 \\
  & RN50 - neutral & 0.863 & 0.881 & 0.573 & 0.480 & 0.677 \\
  & ViT-B/32 - best & 0.856 & 0.855 & 0.624 & 0.378 & 0.741 \\
  & ViT-B/32 - neutral & 0.856 & 0.875 & 0.551 & 0.496 & 0.691 \\
  & ViT-B/16 - best & 0.871 & 0.858 & 0.607 & 0.387 & 0.757 \\
  & ViT-B/16 - neutral & 0.871 & 0.904 & 0.524 & 0.453 & 0.715 \\
\midrule
\multirow{6}{*}{CUB}
  & RN50 - best & 0.666 & 0.603 & 0.513 & 0.711 & 0.497 \\
  & RN50 - neutral & 0.666 & 0.684 & 0.679 & 0.430 & 0.486 \\
  & ViT-B/32 - best & 0.747 & 0.601 & 0.488 & 0.688 & 0.460 \\
  & ViT-B/32 - neutral & 0.747 & 0.690 & 0.682 & 0.363 & 0.603 \\
  & ViT-B/16 - best & 0.750 & 0.654 & 0.392 & 0.699 & 0.506 \\
  & ViT-B/16 - neutral & 0.750 & 0.737 & 0.683 & 0.421 & 0.488 \\
\bottomrule
\end{tabular}
\end{table}

\end{document}